\documentclass{article}

\usepackage{chngcntr}
\usepackage[papersize={8.5in,11in},top=1in,bottom=0.75in,left=0.75in,right=0.75in]{geometry}
\usepackage[round]{natbib}
\usepackage{graphicx}
\usepackage{amsthm,amsmath,amsfonts}
\usepackage[colorlinks,linkcolor=red,anchorcolor=blue,citecolor=blue]{hyperref}
\usepackage{algorithm}
\usepackage{algorithmic}
\usepackage{smile}
\usepackage{authblk}
\usepackage{enumitem}
\newcommand{\R}{\mathbb{R}}
\newcommand{\E}{\mathbb{E}}
\newcommand{\ra}{\rightarrow}

\newcommand{\bbP}{\mathbb{P}}
\newcommand{\tv}{\mathrm{tv}}

\newcommand{\obeta}{\overline{\beta}}
\newcommand{\wt}{\widetilde}

\newcommand{\politex}{{\textsc{Politex}} }

\title{Improved Regret Bound and Experience Replay in Regularized Policy Iteration}

\author[1]{Nevena Lazi\'{c} \thanks{nevena@google.com}}
\author[1]{Dong Yin \thanks{dongyin@google.com}}
\author[1]{Yasin Abbasi-Yadkori \thanks{yadkori@google.com}}
\author[1,2]{Csaba Szepesv\'{a}ri \thanks{szepi@google.com}}
\affil[1]{DeepMind}
\affil[2]{University of Alberta}

\begin{document}
\maketitle

\begin{abstract}
In this work, we study algorithms for learning in infinite-horizon undiscounted Markov decision processes (MDPs) with function approximation.
We first show that the regret analysis of the \textsc{Politex} algorithm (a version of regularized policy iteration) can be sharpened from $O(T^{3/4})$ to $O(\sqrt{T})$ under nearly identical assumptions, and instantiate the bound with linear function approximation.  Our result provides the first high-probability $O(\sqrt{T})$ regret bound for a computationally efficient algorithm in this setting. The exact implementation of \politex with neural network function approximation is inefficient in terms of memory and computation.
Since our analysis suggests that we need to approximate the average of the action-value functions of past policies well, we propose a simple efficient implementation where we train a single Q-function on a replay buffer with past data. We show that this often leads to superior performance over other implementation choices, especially in terms of wall-clock time. Our work also provides a novel theoretical justification for using experience replay within policy iteration algorithms.
\end{abstract}

\section{Introduction}
Model-free reinforcement learning (RL) algorithms combined with powerful function approximation have achieved impressive performance in a variety of application domains over the last decade. Unfortunately, the theoretical understanding of such methods is still quite limited. 
In this work, we study single-trajectory learning in infinite-horizon undiscounted Markov decision processes (MDPs), also known as average-reward MDPs, which capture tasks such as routing and the control of physical systems. 

One line of works with performance guarantees for the average-reward setting follows the ``online MDP'' approach proposed by \citet{even2009online}, where the agent selects policies by running an online learning algorithm in each state, typically mirror descent. The resulting algorithm is a version of approximate policy iteration (API), which alternates between (1) estimating the state-action value function (or Q-function) of the current policy and (2) setting the next policy to be optimal w.r.t. the sum of all previous Q-functions plus a regularizer. Note that, by contrast, standard API sets the next policy only based on the most recent Q-function. 
The policy update can also be written as maximizing the most recent Q-function minus KL-divergence to the previous policy, which is somewhat similar to recently popular versions of API \citep{schulman2015trust,schulman2017proximal,achiam2017constrained,abdolmaleki2018maximum,vmpo}.

The original work of \citet{even2009online} studied this scheme with known dynamics, tabular representation, and adversarial reward functions. More recent works \citep{politex,hao2020provably,wei2020learning} have adapted the approach to the case of unknown dynamics, stochastic rewards, and value function approximation.  With linear value functions, the \politex algorithm of \citet{politex} achieves $O(T^{3/4})$ high-probability regret in ergodic MDPs, and the results only scale in the number of features rather than states. \citet{wei2020learning} later show an $O(\sqrt{T})$ bound on \emph{expected} regret for a  similar algorithm named MDP-EXP2.  In this work, we revisit the analysis of \politex and show that it can be sharpened to $O(\sqrt{T})$ under nearly identical assumptions, resulting in the first $O(\sqrt{T})$ \emph{high-probability} regret bound for a computationally efficient algorithm in this setting.

In addition to improved analysis, our work also addresses practical implementation of \politex with neural networks. The policies produced by \politex in each iteration require access to the sum of all previous Q-function estimates. With neural network function approximation, exact implementation requires us to keep all past networks in memory and evaluate them at each step, which is inefficient in terms of memory and computation. Some practical implementation choices include subsampling Q-functions and/or optimizing a KL-divergence regularized objective w.r.t. a parametric policy at each iteration. We propose an alternative approach, where we approximate the average of all past Q-functions by training a single network on a replay buffer with past data. We demonstrate that this choice often outperforms other approximate implementations, especially in terms of run-time. When available memory is constrained, we propose to subsample transitions using the notion of coresets \citep{bachem2017practical}.

Our work also provides a novel perspective on the benefits of experience replay.
Experience replay is a standard tool for stabilizing learning in modern deep RL, and typically used 
in \emph{off-policy} methods like Deep Q-Networks \citep{mnih2013playing}, as well as ``value gradient'' methods such as DDPG \citep{lillicrap2016continuous} and SVG \citep{heess2015learning}. 
A different line of \emph{on-policy} methods typically does not rely on experience replay; instead, learning is stabilized by constraining consecutive policies to be close in terms of KL divergence  \citep{schulman2015trust,vmpo,degrave2019quinoa}. We observe that both experience replay and KL-divergence regularization can be viewed as approximate implementations of \textsc{Politex}.
Thus, we provide a theoretical justification for using experience replay in API, as an approximate implementation of online learning in each state. Note that this online-learning view  differs from the commonly used justifications for experience replay, namely that it ``breaks temporal correlations'' \citep{schaul2015prioritized, mnih2013playing}. Our analysis also suggests a new objective for subsampling or priority-sampling transitions in the replay buffer, which differs priority-sampling objectives of previous work \citep{schaul2015prioritized}.

In summary, our main contributions are (1) an improved analysis of \textsc{Politex}, showing an $O(\sqrt{T})$ regret bound under the same assumptions as the original work, and (2) an efficient implementation that also offers a new perspective on the benefits of experience replay.

\section{Setting and Notation}

We consider learning in infinite-horizon undiscounted ergodic MDPs $(\cX, \cA, r, P)$, where $\cX$ is the state space,  $\cA$ is a finite action space, $r:\cX\times \cA\to[0, 1]$ is an unknown reward function, and $P:\cX\times \cA \to \Delta_{\cX}$ is the unknown probability transition function. A policy $\pi:\cX\to\Delta_{\cA}$ is a mapping  from a state to a distribution over actions. Let $\{(x_t^{\pi}, a_t^{\pi})\}_{t=1}^{\infty}$
denote the state-action sequence obtained by following
policy $\pi$. The expected average reward of policy $\pi$ is defined as 
\begin{equation}\label{eq:def_jpi}
   J_{\pi}:=\lim_{T\to\infty}\mathbb E\left[\frac{1}{T}\sum_{t=1}^T r(x_t^{\pi}, a_t^{\pi})\right]. 
\end{equation}
Let $\mu_\pi$ denote the stationary state distribution of a policy $\pi$, satisfying $\mu_\pi = \E_{x\sim \mu_{\pi}, a \sim \pi}[P(\cdot|x, a)]$. We will sometimes write $\mu_\pi$ as a vector, and use $\nu_\pi = \mu_\pi \otimes \pi$ to denote the stationary state-action distribution. In ergodic MDPs, $J_\pi$ and $\mu_\pi$ are well-defined and independent of the initial state.
The optimal policy $\pi_*$ is a policy that maximizes the expected average reward.
We will denote by $J_*$ and $\mu_*$ the expected average reward and stationary state distribution of $\pi_*$, respectively.

The value function of a policy $\pi$ is defined as:
\begin{equation}\label{eq:def_v}
V_{\pi}(x) := \mathbb E\left[\sum_{t=1}^{\infty} (r(x_t^{\pi}, a_t^{\pi}) - J_{\pi})|x_1^{\pi} = x\right].
\end{equation}
The state-action value function $Q_{\pi}(x,a)$ is defined as
\begin{equation}\label{eq:def_q}
    Q_{\pi}(x,a) := r(x,a) - J_\pi + \sum_{x'}P(x'|x, a) V_\pi(x').
\end{equation}
Notice that
\begin{equation}\label{eq:v_and_q}
    V_{\pi}(x) = \sum_{a}\pi(a|x)Q_{\pi}(x, a).
\end{equation}
Equations~\eqref{eq:def_q} and~\eqref{eq:v_and_q} are known as the Bellman equation. If we do not use the definition of $V_{\pi}(x)$ in Eq.~\eqref{eq:def_v} and instead solve for $V_{\pi}(x)$ and $Q_{\pi}(x, a)$ using the Bellman equation, we can see that the solutions are unique up to an additive constant. Therefore, in the following, we may use $V_{\pi}(x)$ and $Q_{\pi}(x, a)$ to denote the same function up to a constant. We will use the shorthand notation $Q_\pi(x, \pi') = \E_{a \sim \pi'(\cdot|x)}[Q_{\pi}(x, a)]$; note that $Q_\pi(x, \pi) = V_\pi(x)$.

The agent interacts with the environment as follows:  at each round $t$, the agent observes a state $x_t\in\cX$, chooses an action $a_t \sim \pi_t(\cdot | x_t)$, and receives
a reward $r_t:= r(x_t,a_t)$. The environment then transitions to the next
state $x_{t+1}$ with probability $ P(x_{t+1}|x_t,a_t)$.  Recall that $\pi_*$ is the optimal policy and $J_*$ is its expected average reward. The regret of an algorithm with respect to this fixed policy is
defined as
\begin{equation}\label{def:regret}
    R_T := \sum_{t=1}^T \Big(J_{*} - r(x_t, a_t)\Big).
\end{equation}
The learning goal is to find an algorithm that minimizes the long-term regret $R_T$. 

Our analysis will require the following assumption on the mixing rate of policies.
\begin{assumption}[Uniform mixing]
\label{ass:mixing}
Let $H_\pi$ be the state-action transition matrix of a policy $\pi$. Let $\gamma(H_\pi)$ be the corresponding \emph{ergodicity coefficient}~\citep{seneta1979coefficients}, defined as
\[
\gamma(H_{\pi}) := \max_{z: z^\top {\bf 1}=0, \norm{z}_1=1} \norm{z^\top H_{\pi}}_1.
\]
We assume that there exists a scalar $\gamma < 1$ such that for any policy $\pi$, $\gamma(H_\pi) \leq \gamma < 1$.
\end{assumption}
Assumption~\ref{ass:mixing} implies that for any pair of distributions $\nu, \nu'$, $\norm{(\nu - \nu')^\top H_{\pi}}_1 \leq  \gamma \norm{\nu - \nu'}_1 $; see Lemma~\ref{lem:contraction} in Appendix~\ref{app:prelim} for a proof.

\section{Related Work}

\textbf{Regret bounds for average-reward MDPs.}
Most no-regret algorithms for infinite-horizon undiscounted MDPs are only applicable to tabular representations and model-based
\citep{bartlett09regal,jaksch2010near,ouyang2017learning,%zhang2019regret,
fruit2018efficient,jian2019exploration,talebi2018variance}. 
For weakly-communicating MDPs with diameter $D$, these algorithms nearly achieve the minimax lower bound $\Omega(\sqrt{D|\cX||\cA|T})$ \citep{jaksch2010near} with high probability. 
\citet{wei2020model} provide model-free algorithms with regret bounds in the tabular setting.
In the model-free setting with function approximation, 
the \politex algorithm \citep{politex} achieve $O(d^{1/2}T^{3/4})$ regret in uniformly ergodic MDPs, where $d$ is the size of the compressed state-action space. \citet{hao2020provably} improve these results to $O(T^{2/3})$. 
More recently, \citet{wei2020learning} present three algorithms for average-reward MDPs with linear function approximation. Among these, FOPO achieves $O(\sqrt{T})$ regret but is computationally inefficient, and OLSVI.FH is efficient but obtains $O(T^{3/4})$ regret. The MDP-EXP2 algorithm is computationally efficient, and under similar assumptions as in \citet{politex} it obtains a $O(\sqrt{T})$ bound on \emph{expected regret} (a weaker guarantee than the high-probability bounds in other works). Our analysis shows a high-probability $O(\sqrt{T})$ regret bound under the same assumptions.

{\bf KL-regularized approximate policy iteration.} Our work is also related to approximate policy iteration algorithms which constrain each policy to be close to the previous policy in the sense of KL divergence. This approach was popularized by TRPO \citep{schulman2015trust}, where it was motivated as an approximate implementation of conservative policy iteration \citep{kakade2002approximately}.
Some of the related subsequent works include PPO \citep{schulman2017proximal}, MPO \citep{abdolmaleki2018maximum}, V-MPO \citep{vmpo}, and CPO \citep{achiam2017constrained}.  While these algorithms place a constraint on consecutive policies and are mostly heuristic, another line of research shows that using KL divergence as a regularizer has a theoretical justification in terms of either regret guarantees \citep{politex,hao2020provably,wei2020model,wei2020learning} or error propagation \citep{vieillard2020leverage,vieillard2020momentum}.

{\bf Experience replay.} Experience replay \citep{lin1992self} is one of the central tools for achieving good performance in deep reinforcement learning. While it is mostly used in off-policy methods such as deep Q-learning \citep{mnih2013playing,mnih2015human}, it has also shown benefits in value-gradient methods \citep{heess2015learning,lillicrap2016continuous}, and has been used in some variants of KL-regularized policy iteration \citep{abdolmaleki2018maximum,tomar2020mirror}. Its success has been attributed to removing some temporal correlations from data fed to standard gradient-based optimization algorithms. \citet{schaul2015prioritized} have shown that non-uniform replay sampling based on the Bellman error can improve performance. Unlike these works, we motivate experience replay from the perspective of online learning in MDPs \citep{even2009online} with the goal of approximating the average of past value functions well.

{\bf Continual learning.} Continual learning (CL) is the paradigm of learning a classifier or regressor that performs well on a set of tasks, where each task corresponds to a different data distribution. The tasks are observed sequentially, and the learning goal is to avoid forgetting past tasks without storing all the data in memory. This is quite similar to our goal of approximating the average of sequentially-observed Q-functions, where the data for approximating each Q-function has different distribution. 
In general, approaches to CL can be categorized as regularization-based \citep{kirkpatrick2017overcoming,zenke2017continual,farajtabar2020orthogonal,yin2020sola}, expansion-based \citep{rusu2016progressive}, and replay-based \citep{lopez2017gradient,chaudhry2018efficient,borsos2020coresets}, with the approaches based on experience replay typically having superior performance over other methods.
\section{Algorithm}\label{sec:api}

Our algorithm is similar to the \politex schema~\citep{politex}. In each phase $k$, \politex obtains an estimate $\widehat Q_{\pi_k}$ of the action-value function $Q_{\pi_k}$ of the current policy $\pi_k$, and then sets the next policy using the mirror descent update rule:
\begin{align}
    \pi_{k+1}(\cdot|x) 
    & = \argmax_{\pi} \widehat Q_{\pi_k}(x, \pi) - \eta^{-1} D_{KL}(\pi \parallel \pi_{k}(\cdot|x)) \notag \\
    & = \argmax_{\pi} \sum_{i=1}^k \widehat Q_{\pi_i}(x, \pi) + \eta^{-1} \mathcal{H}(\pi) \notag \\
    & \propto \exp \bigg( \eta \sum_{i=1}^k \widehat Q_{\pi_i}(x, \cdot) \bigg),
    \label{eq:kl_politex}
\end{align}
where $\cH(\cdot)$ is the entropy function. 
When the functions $\{ \widehat Q_{\pi_i}\}_{i=1}^k$ are tabular or linear, the above update can be implemented efficiently by simply summing all the table entries or weight vectors. However, with neural network function approximation, we need to keep all networks in memory and evaluate them at each step, which quickly becomes inefficient in terms of storage and computation. Some of the efficient implementations proposed in literature include keeping a subset of the action-value functions \citep{politex}, and using a parametric policy and optimizing the KL-regularized objective w.r.t. the parameters over the available data \citep{tomar2020mirror}.

Our proposed method, presented in Algorithm~\ref{alg:api}, attempts to directly approximate the average of all previous action-value functions
\begin{align*}
    Q_k(x, a) := \frac{1}{k} \sum_{i=1}^k Q_{\pi_i}(x, a).
\end{align*}
To do so, we only use a single network, and continually train it to approximate $Q_k$. At each iteration $k$, we obtain a dataset of tuples $\cD_k =\{(x_t, a_t, R_t)\}$, where $R_t$ is the empirical return from the state-action pair $(x_t, a_t)$.
We initialize $\widehat Q_k$ to $\widehat Q_{k-1}$ and update it by minimizing the squared error over the union of $\cD_k$ and the replay buffer $\cR$. We then update the replay buffer with all or a subset of data in $\cD_k$.

In the sequel, in Section~\ref{sec:regret}, we first show that by focusing on estimating the average Q-function, we can improve the regret bound of \politex under nearly identical assumption; in Section~\ref{sec:linear}, we instantiate this bound for linear value functions, where we can estimate the average Q-function simply by weight averaging; in Section~\ref{sec:implementation}, we focus on practical implementations, in particular, we  discuss the limitations of weight averaging, and provide details on how to leverage replay data when using non-linear function approximation; in Section~\ref{sec:experiments} we present our experimental results; and in Section~\ref{sec:discussion} we make final remarks.

\begin{algorithm}[t!]
\caption{Schema for policy iteration with replay}
\begin{algorithmic}[1]\label{alg:api}
\STATE \textbf{Input:} phase length $\tau$, num. phases $K$,  parameter $\eta$
\STATE \textbf{Initialize:} $\pi_1(a|x) = 1/|\cA|$, empty replay buffer $\mathcal{R}$
\FOR{$k=1,\ldots, K$}
\STATE 
Execute $\pi_k$ for $\tau$ time steps and collect data $\cD_k$
\STATE Compute $\widehat Q_k$, an estimate of $Q_{k} = \frac{1}{k}\sum_{i=1}^k Q_{\pi_i}$, from data $\cD_k$ and replay $\cR$ 
\STATE Set $\pi_{k+1}(a|x) \propto \exp\big(\eta k \widehat Q_k(x,a)\big)$
\STATE Update replay $\mathcal{R}$ with $\cD_k$
\ENDFOR
\STATE \textbf{Output:} $\pi_{K+1}$
	\end{algorithmic}
\end{algorithm} 

\section{Regret Analysis of \politex}\label{sec:regret}

In this section, we revisit the regret analysis of \politex \citep{politex} in ergodic average-reward MDPs, and show that it can be improved from $O(T^{3/4})$ to $O(\sqrt{T})$ under similar assumptions. Our analysis relies in part on a simple modification of the regret decomposition. Namely, instead of including the estimation error of each value-function $Q_{\pi_k}$ in the regret, we consider the error in the running average $Q_k$. When this error scales as $O(1/\sqrt{k})$ in a particular weighted norm, the regret of \politex is $O(\sqrt{T})$. As we show in the following section, this bound can be instantiated for linear value functions under the same assumptions as \citet{politex}.

\begin{assumption}[Boundedness]
\label{ass:bounded_diff}
Let $\widehat Q_{\pi_k} := k \widehat Q_k - (k-1) \widehat Q_{k-1}$. We assume that there exists a constant $Q_{\max}$ such that for all $k=1, ..., K$ and for all $x \in \cX$,
\[
\max_{a} \widehat Q_{\pi_k}(x, a) - \min_{a} \widehat Q_{\pi_k}(x, a)  \leq Q_{\max}.
\]
\end{assumption}
For the purpose of our algorithm, functions $\widehat Q_{\pi_k}$ are unique up to a constant. Thus we can equivalently assume that $\norm{\widehat Q_{\pi_k}(x, \cdot)}_\infty \leq Q_{\max}$ for all $x$.

Define $\widehat V_{\pi_k}(x) := \widehat Q_{\pi_k}(x, \pi_k)$. Let $V_k := \frac{1}{k} \sum_{i=1}^k V_{\pi_i}(x)$ and $\widehat V_k(x) := \sum_{i=1}^k \widehat V_{\pi_i}(x)$ be the average of the state-value functions and its estimate. We will require the estimation error of the running average to scale as in the following assumption.
\begin{assumption}[Estimation error]
\label{ass:est_err}
Let $\mu_*$ be the stationary state distribution of the optimal policy $\pi_*$.
With probability at least $1-\delta$, for a problem-dependent constant $C$, the errors in $\widehat Q_K$ and $\widehat V_K$ are bounded as
\begin{align*}
   \E_{x \sim \mu_*} [ \widehat V_K(x) -  V_K(x)]
    &\leq C \sqrt{\log(1/\delta) / K} \\   
    \E_{x \sim \mu_* }[ Q_K(x, \pi_*) - \widehat Q_K(x, \pi_*)]
    &\leq C \sqrt{\log(1/\delta) / K} \,.
\end{align*}
\end{assumption}

Define $S_\delta(|\cA|, \mu_*)$ as in \citet{politex}:
\begin{align*}
S_\delta(|\cA|, \mu_*) := \sqrt{\frac{\log|\cA|}{2}} + \big\langle \mu_*, \sqrt{\frac{1}{2}\log\frac{1}{\delta\mu_*}}\big\rangle .
\end{align*}
We bound the regret of \politex in the following theorem. Here, recall that $\tau$ is the length of each phase, and $\gamma$ and $\eta$ are defined in Assumption~\ref{ass:mixing} and Eq.~\eqref{eq:kl_politex}, respectively.
\begin{theorem}[Regret of \textsc{Politex}]
\label{thm:regret}
Let Assumptions \ref{ass:mixing}, \ref{ass:est_err}, and \ref{ass:bounded_diff} hold. 
%Let $S_\delta(|\cA|, \mu_*)= \sqrt{\ln|\cA|} + \langle \mu_*, \sqrt{(\log(2/\delta) + \log(1/\mu_*))}\rangle$. 
For $\tau \geq \frac{\log T}{2 \log(1/\gamma)}$ and $\eta = \frac{\sqrt{8\log|\cA|}}{Q_{\max}\sqrt{K}}$, for a  constant $C_1$,  with probability  at least $1-4\delta$, the regret of \politex in ergodic average-reward MDPs is bounded as 
\begin{align*}
    R_T \leq \frac{C_1 (1+ Q_{\max})  S_\delta(|\cA|, \mu_*)\sqrt{\tau}}{(1-\gamma)^{2}} \sqrt{T}\,.
\end{align*}
\end{theorem}

\begin{proof}
We start by decomposing the cumulative regret, following similar steps as \citet{politex}:
\begin{equation}\label{eqn:regret_dec}
R_T = \sum_{k=1}^K \sum_{t=(k-1)\tau+1}^{k\tau} (J_* - J_{\pi_k}) + (J_{\pi_k} - r_t).
\end{equation}
The second term $V_T = \sum_{k=1}^K \sum_{t=(k-1)\tau+1}^{k\tau}(J_{\pi_k} -r_t)$ captures the sum of differences between observed rewards and their long term averages. In previous work, this term was shown to scale as $O(K\sqrt{\tau})$. We show that the analysis  can in fact be tightened to $O(\sqrt{T})$ using improved concentration bounds and the slow-changing nature of the policies. See  Lemma~\ref{lem:vtbound} in Appendix~\ref{sec:vtbound} for precise details.

The first term, which is also called \emph{pseudo-regret} in literature, measures the difference between the expected reward of the reference policy and the policies produced by the algorithm.
Applying the performance difference lemma~\citep{cao1999single}, we can write each pseudo-regret term as
\begin{align*}
    J_{*} -J_{\pi_k} &= \E_{x \sim \mu_{*}} \left[ Q_{\pi_{k}}(x, \pi_{*})-Q_{\pi_{k}}(x, \pi_k) \right].
\end{align*}

Now, notice that \politex is running exact mirror descent for loss functions $\widehat{Q}_{\pi_k}$. We bridge the pseudo-regret by the $\widehat{Q}_{\pi_k}$ terms:
\begin{align}
        R_{T1a}&= \tau\sum_{k=1}^K \E_{x \sim \mu_{*}} \left[ Q_{\pi_{k}}(x, \pi_{*})- \widehat{Q}_{\pi_k}(x, \pi_*) \right]  \label{eq:rt1a} \\
        R_{T1b} &= \tau\sum_{k=1}^K \E_{x \sim \mu_{*}} \left[ \widehat{Q}_{\pi_k}(x, \pi_k)-Q_{\pi_{k}}(x, \pi_k) \right] \label{eq:rt1b} \\
 R_{T2}&=
\tau\sum_{k=1}^K \E_{x \sim \mu_{*}} \left[ \widehat{Q}_{\pi_k}(x, \pi_*)-\widehat{Q}_{\pi_k}(x, \pi_k)\right]. \label{eq:rt2}
\end{align}
$R_{T2}$ can be bounded using the regret of mirror descent as in previous work \citep{politex}. Setting $\eta = \frac{\sqrt{ \log |\cA| }}{Q_{\max}\tau\sqrt{2K}}$, and using a union bound over all states, with probability at least $1-\delta$, $R_{T2}$ is bounded as
\begin{align*}
    R_{T2} \leq \tau Q_{\max} S_{\delta}(|\cA|, \mu_*) \sqrt{K}  \,.
\end{align*}

{\bf Bounding regret due to estimation error.} We now focus on bounding $R_{T1} = R_{T1a} + R_{T1b}$ under Assumption~\ref{ass:est_err}. We have the following:
\begin{align*}
   R_{T1a} &= \tau \sum_{k=1}^K \E_{x \sim \mu_*} \big[ Q_{\pi_k}(x, \pi_*) - \widehat{Q}_{\pi_k}(x, \pi_*) \big] \\
   &= (\tau K) \E_{x \sim \mu_*, a \sim \pi_*}\left[Q_K(x, a) - \widehat Q_K(x, a)\right] \\
   & \leq C \tau \sqrt{K\log(1/\delta) }.\\
R_{T1b} &= \tau \sum_{k=1}^K \E_{x \sim \mu_*} \big[\widehat V_{\pi_k}(x) - V_{\pi_k}(x) \big]  \\
&= \tau K \E_{x \sim \mu_*} \big[ \widehat V_K(x) - V_K(x) \big] \\
& \leq C \tau \sqrt{K\log(1/\delta)}.
\end{align*}
We can then obtain the final result by combining the bounds on $R_{T1a}$, $R_{T1b}$, $R_{T2}$, $V_T$ from Appendix~\ref{sec:vtbound}, and using union bound as well as the fact that $K = T / \tau$.
\end{proof}
\section{Linear Value Functions}\label{sec:linear}

In this section, we show that the estimation error condition in Assumption~\ref{ass:est_err} (and thus $O(\sqrt{T})$ regret) can be achieved under similar assumptions as in \citet{politex} and \citet{wei2020learning}, which we state next.

\begin{assumption}[Linear value functions]
\label{ass:linear}
The action-value function $Q_\pi$ of any policy $\pi$ is linear: 
$Q_\pi(x, a) = w_\pi^\top \phi(x, a)$,
where $\phi: \cS \times \cA \rightarrow \R^d$ is a known feature function such that $\max_{x, a} \norm{\phi(x, a)} \leq C_\Phi$.
\end{assumption}

\begin{assumption}[Feature excitation]
\label{ass:excite}
There exists a constant $\sigma^2$ such that for any policy $\pi$,
\[
\lambda_{\min} \left( \E_{(x, a) \sim \mu_\pi \otimes \pi} [\phi(x, a) \phi(x, a)^\top] \right) \geq \sigma^2 > 0.
\]
\end{assumption}

We now describe a simple procedure for estimating the average action-value functions $Q_k(x, a) = \phi(x, a)^\top w_k$, where $w_k = \frac{1}{k} \sum_{i=1}^k w_{\pi_i}$, such that the conditions of Assumption~\ref{ass:est_err} are satisfied. Essentially, we estimate each $Q_{\pi_i}$ using least-squares Monte Carlo and then average the weights. We will use the shorthand notation  $\phi_t = \phi(x_t, a_t)$. Let $\cH_i$ and $\cT_i$ be subsets of time indices in phase $i$ (defined later). 
We estimate  $Q_{\pi_i}$ as follows:
\begin{align}
    \widehat w_{\pi_i} 
    &= \bigg(\sum_{t \in \cH_i} \phi_t \phi_t^\top + \alpha I\bigg)^{-1} \sum_{t \in \cH_i} \phi_t R_t
\end{align}
where $R_t$ are the empirical $b$-step returns ($b$ is specified later), computed as
\begin{align}\label{eq:linear_reward}
R_t = \sum_{j=t}^{t+b} (r_t - \widehat J_{\pi_i}), \quad
    \widehat J_{\pi_i} = \frac{1}{|\cT_i|} \sum_{t \in \cT_i} r_t. 
\end{align}
We then estimate $w_k$ as $\widehat w_k = \frac{1}{k} \sum_{i=1}^k \widehat w_{\pi_i}$. Note that for this special case of linear value functions, we do not need to use the replay buffer in Algorithm~\ref{alg:api}. For analysis purposes, we divide each phase of length $\tau$ into $2m$ blocks of size $b$ and let $\cH_i$ ($\cT_i$) be the starting indices of odd (even) blocks in phase $i$. Due to the gaps between indices and fast mixing, this makes the data almost independent (we make this precise in Appendix~\ref{app:linear}) and the error easier to analyze. In practice, one may simply want to use all data. 

For a distribution $\mu$, let $\norm{x}_\mu$ denote the distribution-weighted norm such that $\norm{x}_\mu^2 = \sum_i \mu_i x_i^2$. Using Jensen's inequality, we have that $(\E_{x\sim \mu}[q(x)])^2 \leq \E_{x \sim \mu}[q(x)^2]$.  Thus, it suffices to bound the $Q$-function error in the distribution-weighted norm, $\norm{Q_K - \widehat Q_K}_{\mu_* \otimes \pi_*}$. Furthermore, given bounded features,
\begin{align*}
    \norm{\widehat Q_K - Q_K}_{\mu_* \otimes \pi_*} 
    &\leq C_{\Phi}\norm{\widehat w_K - w_K}_2 \,,
\end{align*}
so it suffices to bound the error in the weights. 
We bound this error in the following Lemma, proven in Appendix~\ref{app:linear}.
\begin{lemma}[Estimation error for linear functions]
\label{lem:linear_est}
Suppose that Assumptions \ref{ass:mixing}, \ref{ass:linear} and \ref{ass:excite} hold and that true action-value weights are bounded as $\norm{w_{\pi_i}}_2 \leq C_w$ for all $i =1,\ldots K$.
Then for any policy $\pi$, for $\alpha = \sqrt{\tau/K}$, $m \geq 72 C_{\Phi}^4 \sigma^{-2} (1 -\gamma)^{-2} \log (d/\delta)$, and $b \geq \frac{\log (T\delta^{-1}(1-\gamma)^{-1} ) }{ \log(1/\gamma)}$, there exists an absolute constant $c$ such that with probability at least $1-\delta$,
\begin{align*}
      \norm{\widehat w_K - w_K}_2 \leq  c \sigma^{-2}(C_w + C_{\Phi})b \sqrt{\frac{ \log(2d/\delta)}{Km}}.
\end{align*}
\end{lemma}
Furthermore, in Appendix~\ref{app:linear_vpi}, we show that the error in the average state-value function satisfies the following: 
\begin{align*}
     \E_{\mu_*} [\widehat V_K(x) - V_K(x)] \leq 
    cC_{\Phi}|\cA|(C_w  + C_{\Phi})\frac{b}{\sigma^2} \sqrt{\frac{ \log(2d/\delta)}{Km}}.
\end{align*}
We have demonstrated that Assumption~\ref{ass:est_err} can be satisfied with linear value functions. 
For Assumption~\ref{ass:bounded_diff}, it suffices for the weight estimates $\{\widehat w_{\pi_i}\}$ to be bounded. This will be true, since we assume that the true weights are bounded and we can bound the error in the weight space. Thus \politex has an $O(\sqrt{T})$ regret in this setting (though note that we incur an extra $|\cA|$ factor coming from the $\widehat V_K$ error bound). 
\section{Practical Implementation}\label{sec:implementation}
As mentioned earlier, the key idea in our policy update is to obtain an estimate $\widehat{Q}_k$ of the average of all the Q-functions in previous phases. We have seen that when we use linear functions to approximate the Q-functions, we can simply average the weights in order to get an estimate of the average Q-function. However, in practice, we often need to use non-linear functions, especially neural networks, to approximate complex Q-functions. In this section, we discuss how to efficiently implement our algorithm with non-linear function approximation.

\subsection{Weight Averaging}\label{sec:weight_ave}
The simplest idea may be averaging the weights of neural networks. However, a crucial difference from the linear setting is that averaging the weights of the neural networks is not equivalent to averaging the functions that they represent: the function that a neural network represent is invariant to the permutation of the hidden units, and thus two networks with very different weights can represent similar functions. Therefore, this implementation may only succeed when all the Q-function approximations are around the same local region in the weight space. Thus, when we learn the new Q-function $\widehat Q_{\pi_k}$ in phase $k$, we should initialize it with $\widehat Q_{k-1}$, run SGD with new data, and then average with $\widehat Q_{k-1}$.

\subsection{Experience Replay}\label{sec:replay}
Another natural idea is to leverage the data from replay buffer to obtain an estimate of the average Q-function. We elaborate the details below. We use simple $b$-step Monte Carlo estimate for the $Q$ value of each state-action pair. For any $(i-1)\tau + 1 \le t \le i\tau - b $, we can estimate the state-action value of $(x_t, a_t)$ by $b$-step cumulative reward\footnote{This is a practical implementation of Eq.~\eqref{eq:linear_reward}, i.e., we do not split the data in each phase into blocks.}
\[
R_t = \sum_{j = t}^{t+b} (r_j - \widehat J_{\pi_i}), \quad \widehat J_{\pi_i} = \frac{1}{\tau} \sum_{j=(i-1)\tau+1}^{i\tau} r_j.
\]
In the following, we denote by $\tau':=\tau-b$ the maximum number of data that we can collect from every phase. At the end of each phase, we store all or a subset of the $(x_t, a_t, R_t)$ tuples in our replay buffer $\mathcal{R}$.
We extract feature $\phi(x, a)\in\R^d$ for the state-action pair $(x, a)$ and let $\mathcal{F}\subseteq\{f:\R^d\mapsto\R \}$ be a class of functions that we use to approximate the Q-functions. For phase $i$, we propose the following method to estimate $Q_{\pi_i}$: $\widehat Q_{\pi_i}(x, a)=\widehat f(\phi(x, a))$, where $\widehat f \in \arg\min_{f\in\mathcal{F}} \ell_i(f)$ and
\begin{equation}\label{eq:squared}
\ell_i(f):=\frac{1}{\tau'}\sum_{t=(i-1)\tau + 1}^{i\tau - b}(f(\phi(x_t, a_t)) - R_t)^2.
\end{equation}
Suppose that we store all the data from the previous phases in the replay buffer, then in order to estimate the average of the Q-functions of the first $k$ phases, i.e., $\widehat Q_k$, we propose to use the heuristic that minimizes the average of the $k$ squared loss functions defined in Eq.~\eqref{eq:squared}, i.e., $\frac{1}{k}\sum_{i=1}^k \ell_i(f)$.

{\bf Subsampling and coreset.}
In practice, due to the high memory cost, it may be hard to store all the data from the previous phases in the replay buffer. We found that a simple strategy to resolve this issue is to begin with storing all the data from every phase, and add a limit on size of the replay buffer. When the buffer size exceeds the limit, we eliminate a subset of the data uniformly at random.

Another approach is to sample a subset of size $s$ from the data collected in each phase. Denote this subset by $\mathcal{R}_i$ for phase $i$. Thus in the $k$-th phase, we have $\tau'$ data $\mathcal{D}_k$ from the current phase as well as $s(k-1)$ data from the replay buffer $\mathcal{R}$. First, suppose that the $s$ data points are sampled uniformly at random. We can then minimize the following objective:
\begin{equation}\label{eq:uniform_downsample}
    \min_{f\in\mathcal{F}}\frac{1}{k}\Big(\ell_i(f) + \sum_{i=1}^{k-1}\widehat{\ell}_i(f)\Big),
\end{equation}
where $\widehat{\ell}_i(f) := \frac{1}{s}\sum_{(x_t, a_t, R_t)\in\mathcal{R}_i} (f(\phi(x_t, a_t)) - R_t)^2$ is an unbiased estimate of $\ell_i(f)$. Further, uniform sampling is not the only way to construct an unbiased estimate. In fact, for any discrete distribution, with PMF $q=\{q_t\}$, over the $\tau'$ data in $\mathcal{D}_i$, we can sample $s$ data points according to $q$ and construct
\begin{equation}\label{eq:hatell_q}
    \widehat{\ell}_i(f) := \frac{1}{\tau'}\sum_{(x_t, a_t, R_t)\in\mathcal{R}_i} \frac{1}{q_t} (f(\phi(x_t, a_t)) - R_t)^2,
\end{equation}
in order to obtain an unbiased estimate of $\ell_i(f)$. As shown by~\citet{bachem2017practical}, by choosing $q_t \propto (f(\phi(x_t, a_t)) - R_t)^2$, we can minimize the variance of $\widehat{\ell}_i(f)$ for any fixed $f$. In the following, we call the subset of data sampled according to this distribution a \emph{coreset} of the data. In the experiments in Section~\ref{sec:experiments}, we show that thanks to the variance reduction effect, sampling a coreset often produces better performance than sampling a subset uniformly at random, especially when the rewards are sparse.

{\bf Comparison to \citet{politex}.}
Our algorithm can be considered as an efficient implementation of \politex~\citep{politex} via experience replay. In the original \politex algorithm, the Q-functions are estimated using Eq.~\eqref{eq:squared} for each phase, and all the functions are stored in memory. When an agent interacts with the environment, it needs to evaluate all the Q-functions in order to obtain the probability of each action. This implies that the time complexity of computing action probabilities increases with the number of phases, and as a result the algorithm is hard to scale to a large number of phases. Although we can choose to evaluate a random subset of the Q-functions to estimate the action probabilities in \textsc{Politex}, our implementation via experience replay can still be faster since we only need to evaluate a single function, trained with replay data, to take actions.

\section{Experiments}\label{sec:experiments}
In this section, we evaluate our implementations empirically. We make comparisons with several baselines in two control environments.

{\bf Environments.}
We use two control environments with the simulators described in~\citet{tassa2018deepmind}. Both environments are episodic with episode length $1000$.
The environments we evaluate are: 
\begin{itemize}
    \item \emph{Cart-pole}~\citep{barto1983neuronlike}: The goal of this environment is to balance an unactuated pole by applying forces to a cart at its base. We discretize the continuous force to $5$ values: $\{-2, -1, 0, 1, 2\}$.
    The reward at each time step is a real number in $[0, 1]$. We end episodes early when the pole falls (we use rewards less than $0.5$ an indicator for falling), and assume zero reward for the remaining steps when reporting results.
    \item \emph{Ball-in-cup}: Here, an actuated planar receptacle can translate in the vertical plane in order to swing and catch a ball attached to its bottom. This task has a sparse reward: $1$ when the ball is in the cup, and $0$ otherwise. We discretize the two-dimensional continuous action space to a $3\times 3$ grid, i.e., $9$ actions in total.
\end{itemize}

\begin{figure*}[ht]
\centering
\includegraphics[width=0.32\textwidth]{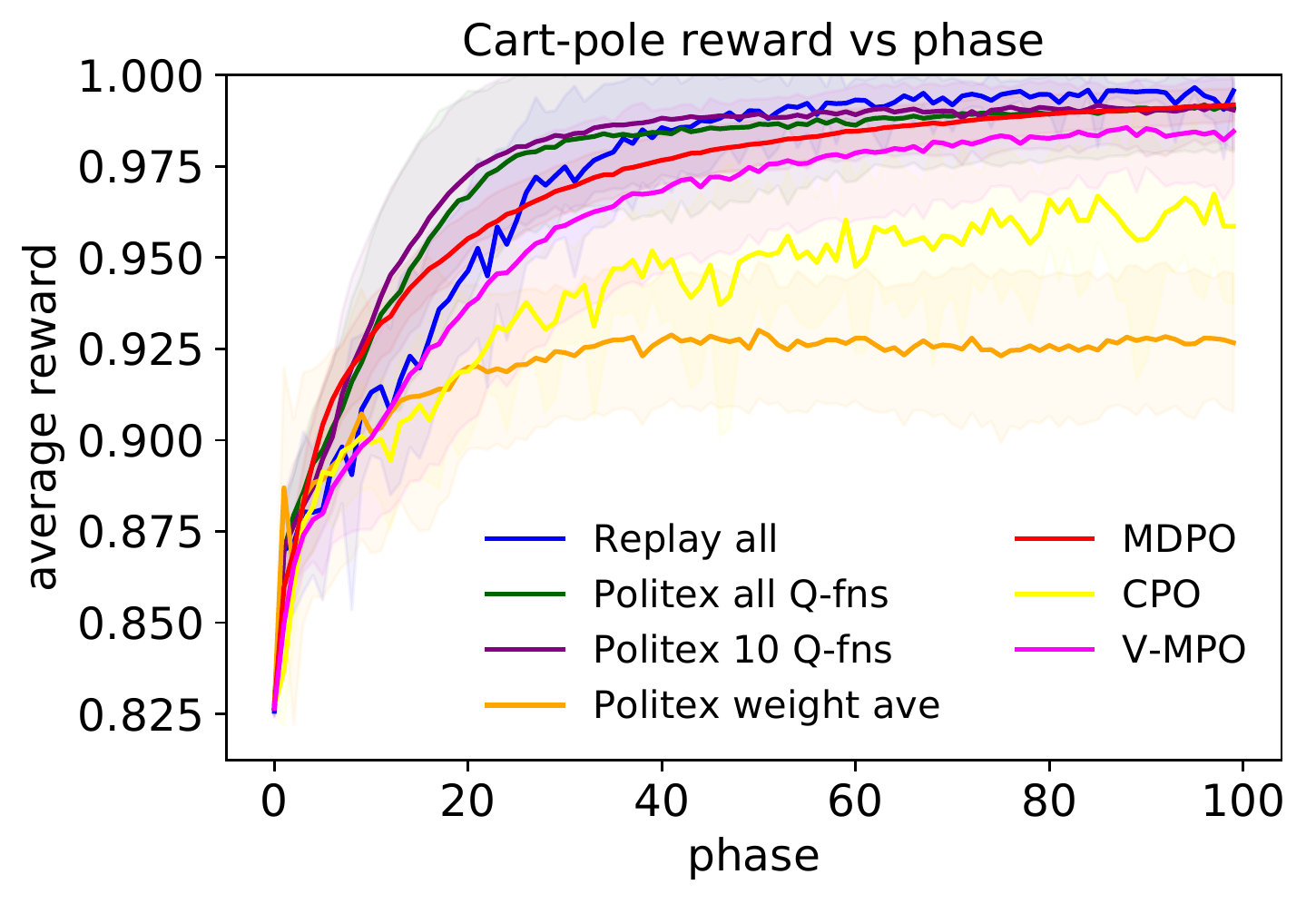}
{\includegraphics[width=0.32\textwidth]{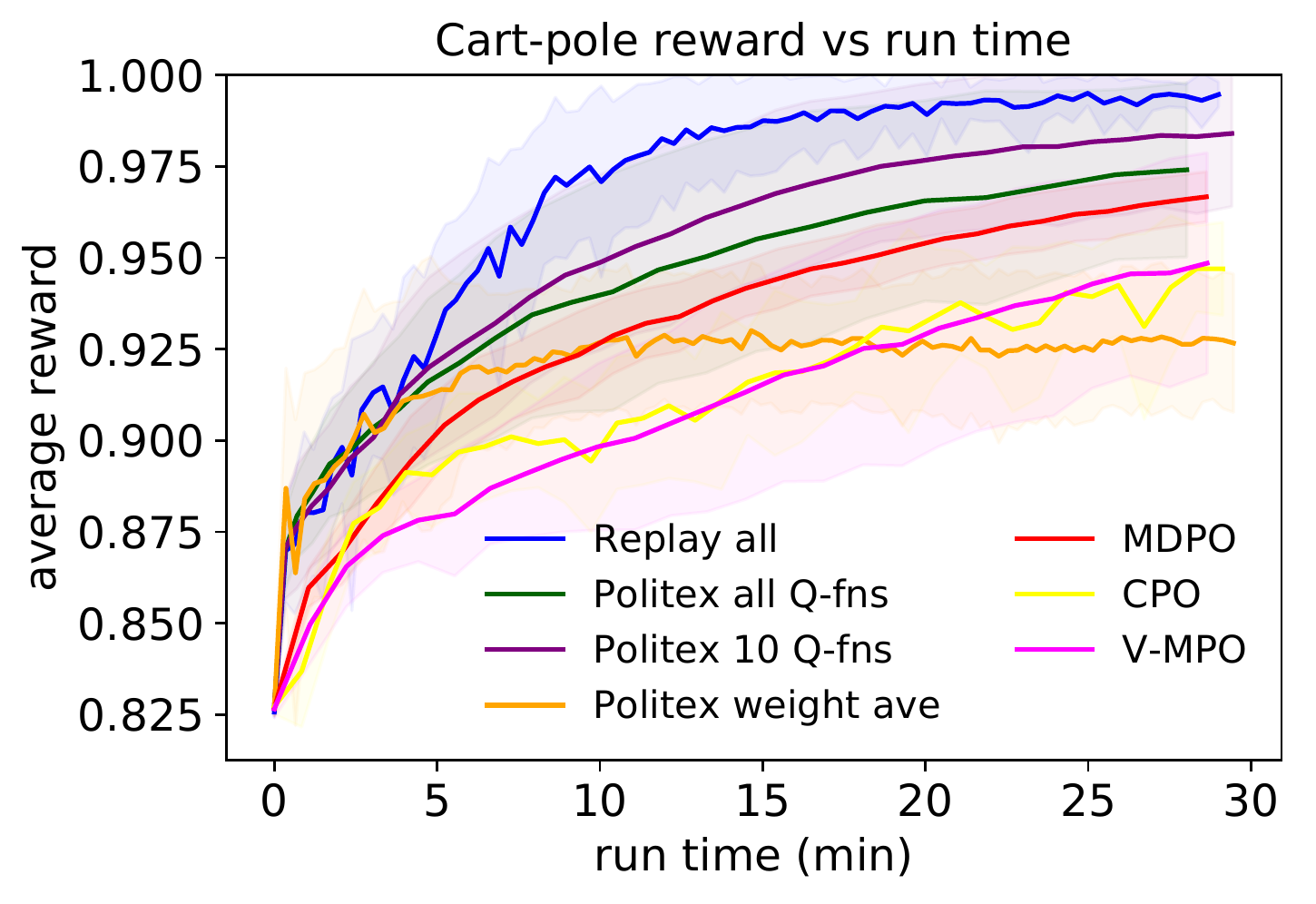}}
{\includegraphics[width=0.32\textwidth]{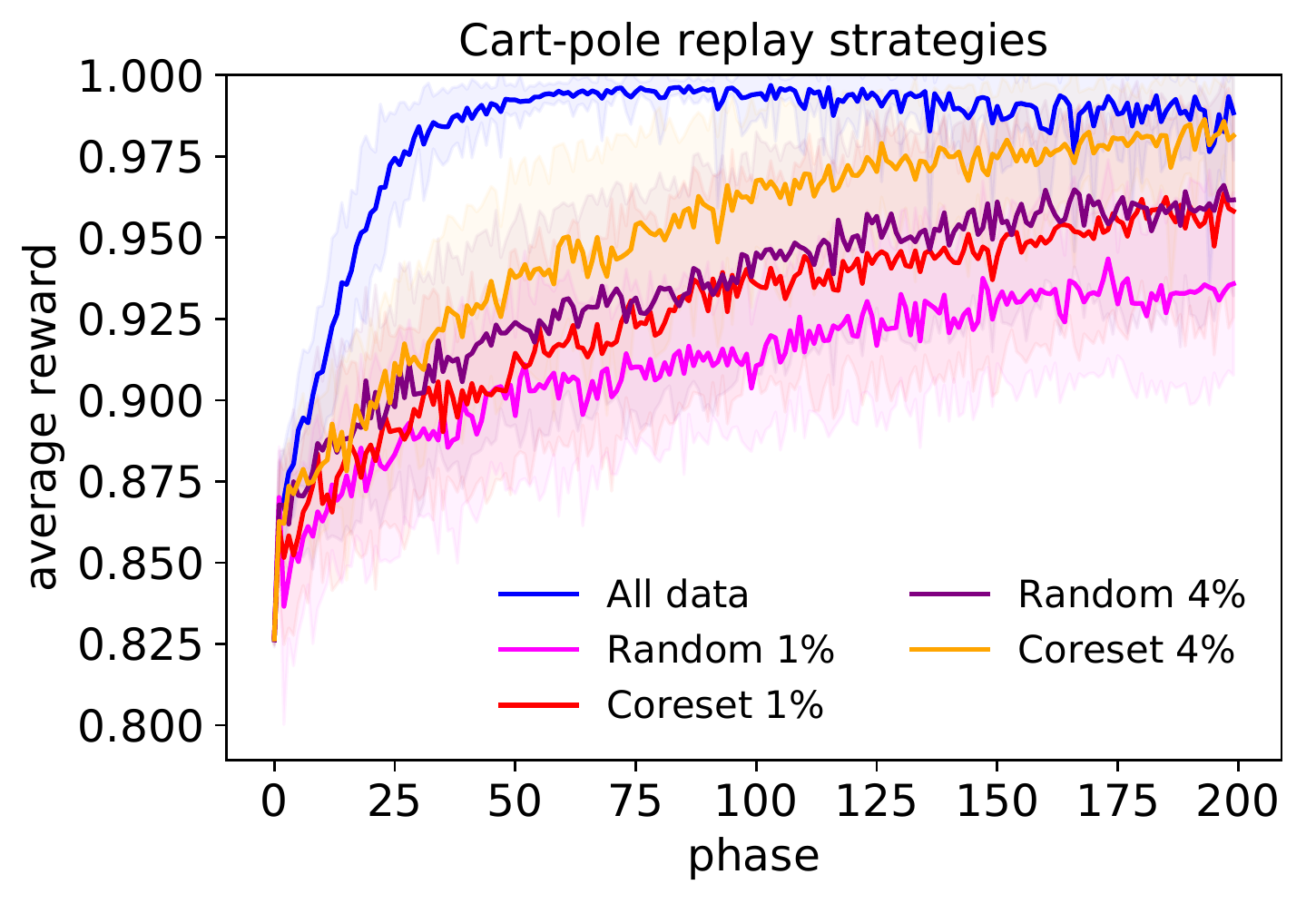}}
\subfigure[Iteration complexity] {\includegraphics[width=0.32\textwidth]{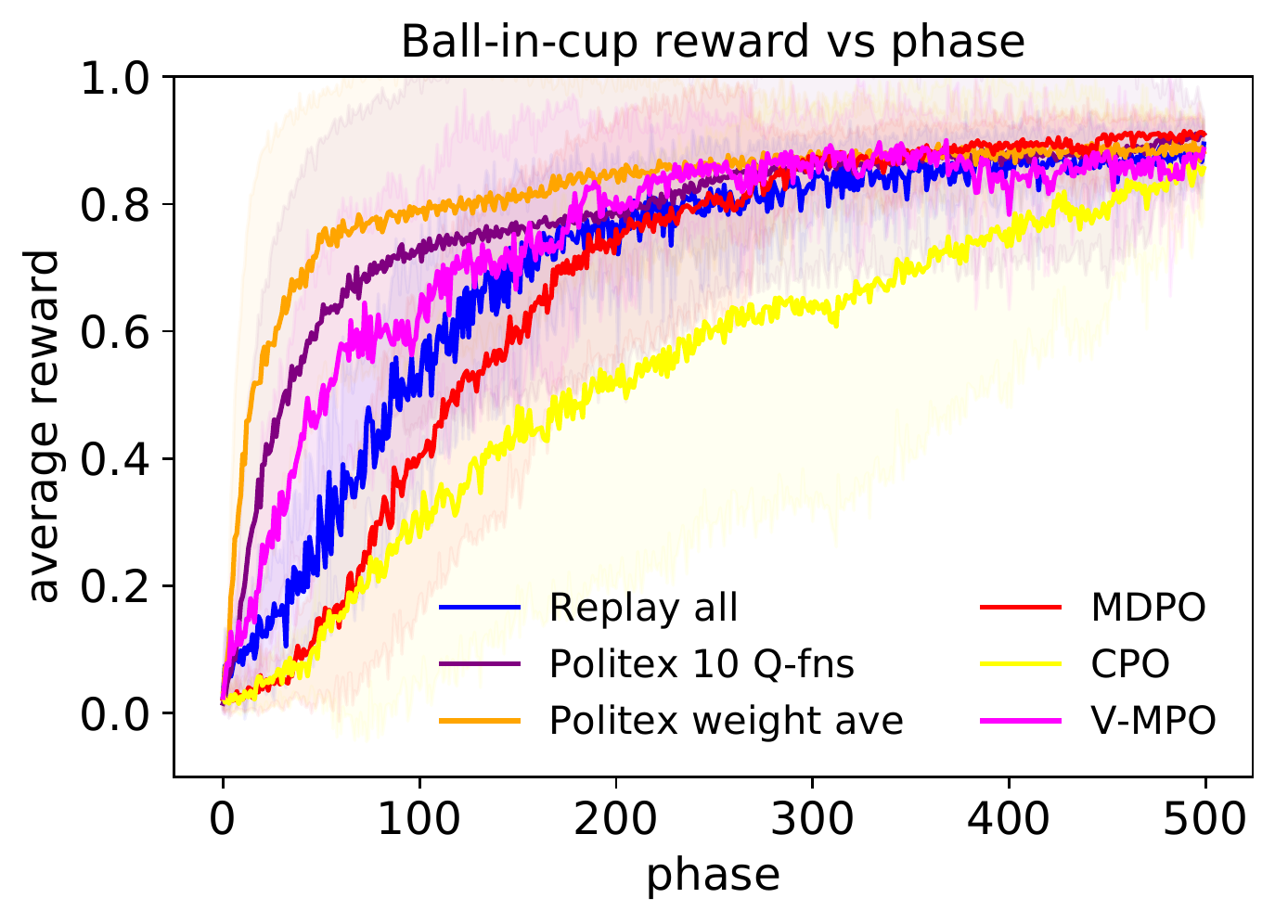}}
\subfigure[Training speed]
{\includegraphics[width=0.32\textwidth]{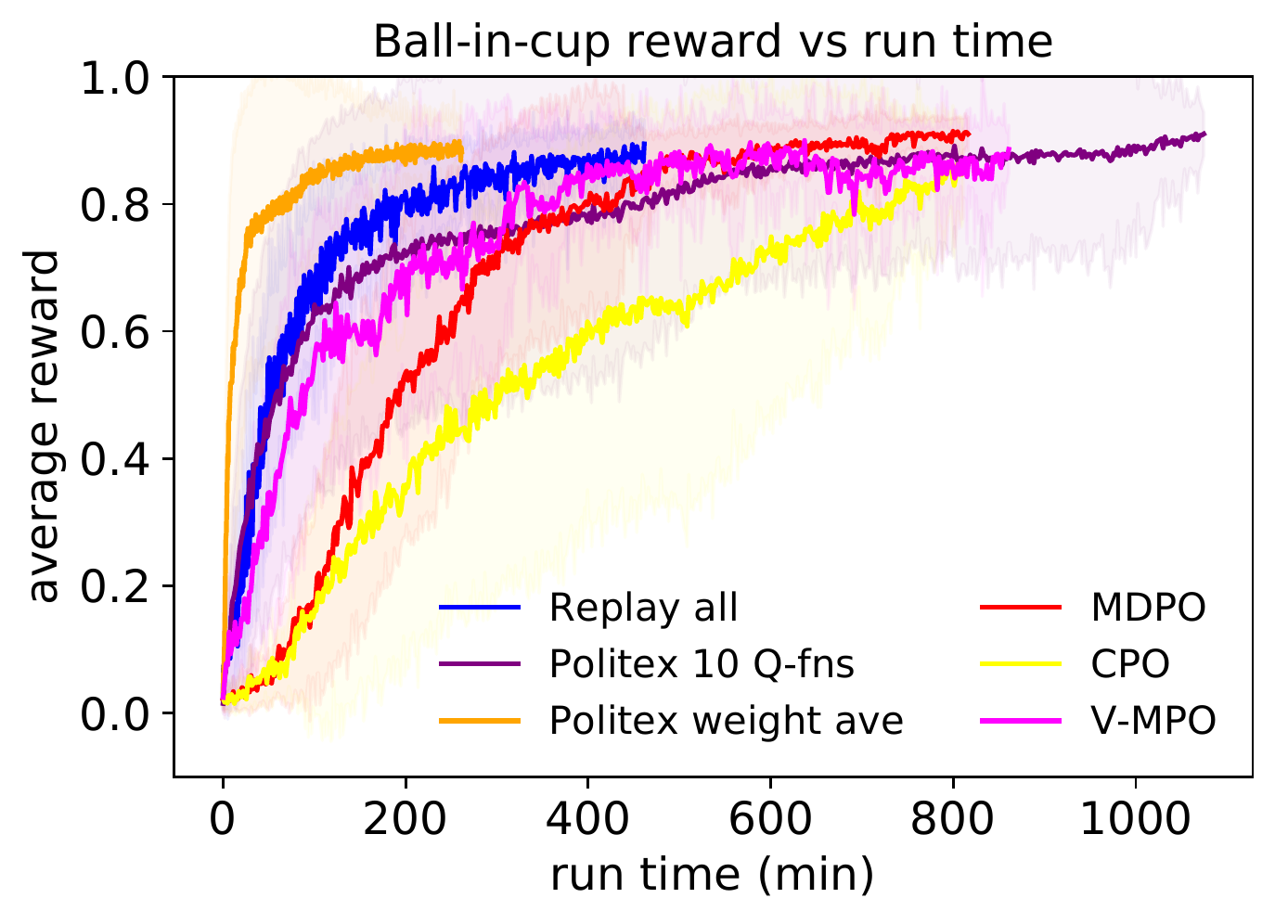}} 
\subfigure[Replay subsampling strategies]{\includegraphics[width=0.32\textwidth]{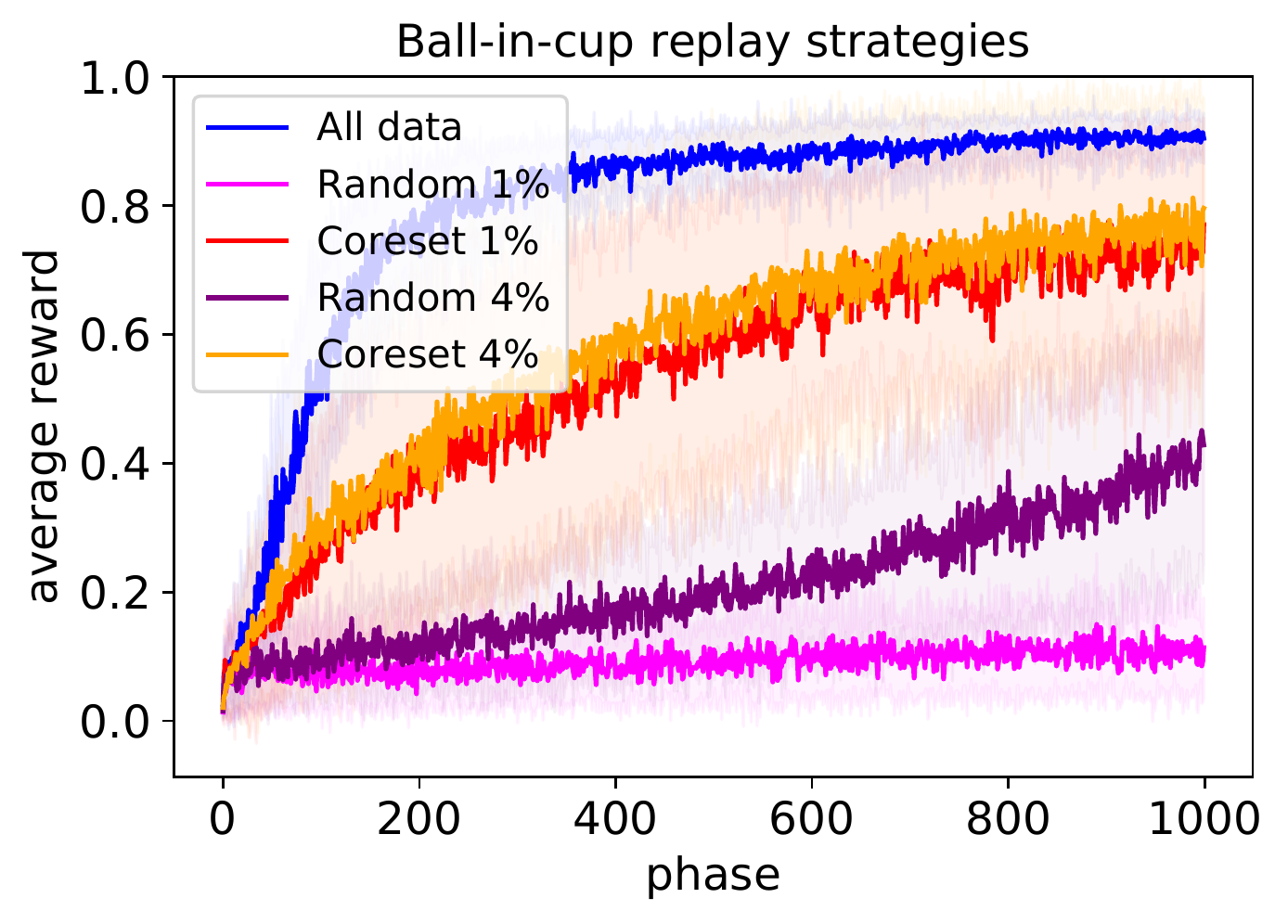}} 
\caption{Experiments on the Cart-pole environment (top) and Ball-in-cup (bottom).}
\label{fig:experiments}
\end{figure*}

{\bf Algorithms.}
We compare the following algorithms: our proposed implementation of \politex using experience replay, \politex using weight averaging, the original \politex algorithm, the variant of \politex that averages $10$ randomly selected Q-functions, the mirror descent policy optimization (MDPO) algorithm~\citep{tomar2020mirror}, the constrained policy optimization (CPO)~\citep{achiam2017constrained}, and the V-MPO algorithm~\citep{song2019v}.

For all the algorithms, we extract Fourier basis features~\citep{konidaris2011value} from the raw observations for both environments: for Cart-pole, we use $4$ bases and for Ball-in-cup we use $2$ bases. For variants of \politex we construct the state-action features $\phi(x, a)$ by block one-hot encoding, i.e., we partition the $\phi(x, a)$ vector into $|\mathcal{A}|$ blocks, and set the $a$-th block to be the Fourier features of $x$ and other blocks to be zero.
We approximate $Q$-functions using neural networks with one hidden layer and ReLU activation: the width of the hidden layer is $50$ for Cart-pole and $250$ for Ball-in-cup.
For MDPO, CPO and V-MPO, the algorithms use a policy network whose input and output are the state feature and action probability, respectively. We use one hidden layer networks with width $50$ for Cart-pole and $250$ for Ball-in-cup. These three algorithms also need a value network, which takes the state feature as input and outputs its value for the current policy. For both environments, we use one hidden layer network with width $50$.

For Cart-pole, we choose phase length $\tau=10^4$ and for Ball-in-cup, we choose $\tau=2\times 10^4$. Since the environments are episodic, we have multiple episodes in each phase. For Cart-pole, since the training usually does not need a large number of phases, we do not set limit on the size of the replay buffer, whereas for Ball-in-cup, we set the limit as $2\times 10^6$. For our algorithm and the variants of \textsc{Politex}, we choose parameter $\eta$ in $\{5, 10, 20, 40, 80, 160\}$ and report the result of each algorithm using the value of $\eta$ that produces the best average reward. For MDPO, CPO and V-MPO, we note that according to Eq.~\eqref{eq:kl_politex}, the KL regularization coefficient between policies is the reciprocal of the $\eta$ parameter in our algorithm, we also run these baseline algorithms in the same range of $\eta$ and report the best result. In addition, in CPO, we set the limit in KL divergence between adjacent policies $0.001\eta$, which, in our experiments, leads to good performance. We treat the length of Monte Carlo estimate $b$ as a hyper parameter, so we choose it in $\{100, 300\}$ and report the best performance.

{\bf Results.} We run each algorithm in each environment $20$ times and report the average results as well as the standard deviation of the $20$ runs (shaded area). We report the average reward in each phase before training on the data collected during the phase. The results are presented in Figure~\ref{fig:experiments}. Every run is conducted on a single P100 GPU.

From Fig.~\ref{fig:experiments}(a), we can see that the iteration complexity (average reward vs number of phases) and final performance of the experience replay implementation using all the data from past phases (blue curve) are similar to many state-of-the-art algorithms, such as \politex and V-MPO. Meanwhile, according to Fig.~\ref{fig:experiments}(b), the experience replay implementation achieves strong performance in training speed, i.e., best performance at $30$ minute in Cart-pole and second best performance at $200$ minute in Ball-in-cup.
We note here that the training time includes both the time for data collection, i.e., interacting with the environment simulator and the training of a new policy. The observation that the experience replay based implementation is faster than the original \politex and the implementation that uses $10$ Q-functions can be explained by the fact that the replay-based algorithm only uses a single Q-function and thus is faster at data collection, as discussed in Section~\ref{sec:replay}. In addition, that the replay-based algorithm runs faster than MDPO, CPO, and V-MPO can be explained by its simplicity: in variants of \textsc{Politex}, we only need to approximate the Q-function, whereas in MDPO, CPO, and V-MPO, we need to train both the policy and value networks.

The \politex weight averaging scheme achieves 
% As for the weight averaging scheme for \textsc{Politex}, we observe that although it achieves 
the best performance on Ball-in-cup in both iteration complexity and training speed, but does not converge to a solution that is sufficiently close to the optimum on Cart-pole. Notice that for weight averaging, we used the initialization technique mentioned in Section~\ref{sec:weight_ave}. Considering the consistent strong performance of experience replay in both environments, we still recommend applying experience replay when using non-linear function approximation.

In Fig.~\ref{fig:experiments}(c), we can see that when we only sample a small subset ($1\%$ or $4\%$) of the data to store in the replay buffer, the coreset technique described in Section~\ref{sec:replay} achieves better performance due to the variance reduction effect. This improvement is quite significant in the Ball-in-cup environment, which has sparse rewards. Notice that sampling coreset only adds negligible computational overhead during training, and thus we recommend the coreset technique when there exists a strict memory constraint.
\section{Discussion}\label{sec:discussion}

The main contributions of our work are an improved analysis and practical implementation of \textsc{Politex}.
On the theoretical side, we show that \politex obtains an $O(\sqrt{T})$ high-probability regret bound in uniformly mixing average-reward MDPs with linear function approximation, which is the first such bound for a computationally efficient algorithm. The main limitation of these result is that, similarly to previous works, they hold under somewhat strong assumptions which circumvent the need for exploration. An interesting future work direction would be to relax these assumptions and incorporate explicit exploration instead.

On the practical side, we propose an efficient implementation with neural networks that relies on experience replay, a standard tool in modern deep RL. 
Our work shows that experience replay and KL regularization can both be viewed as approximately implementing mirror descent policy updates within a policy iteration scheme. This provides an online-learning  justification for using replay buffers in policy iteration, which is different than the standard explanation for their success.
Our work also suggests a new objective for storing and prioritizing replay samples, with the goal of approximating the average of value functions well. This goal has some similarities with continual learning, where experience replay has also lead to empirical successes. One interesting direction for future work would be to explore other continual learning techniques in the approximate implementation of mirror descent policy updates.

\section*{Acknowledgements}
We would like to thank Mehrdad Farajtabar, Dilan G\"or\"ur, Nir Levine, and Ang Li for helpful discussions.

\bibliographystyle{abbrvnat}
\bibliography{arxiv/references.bib}
\newpage
\appendix
\section*{Appendix}
\section{Preliminaries}
\label{app:prelim}

We state some useful definitions and lemmas in this section. 
\begin{lemma}
\label{lem:contraction}
Let $X_1$ and $X_2$ be a pair of distribution vectors. Let $H$ be the transition matrix of an ergodic Markov chain with a stationary distribution $\nu$, and ergodicity coefficient (defined in Assumption~\ref{ass:mixing}) upper-bounded by $\gamma < 1$. Then
\begin{align*}
    \norm{ (H^m)^\top (X_1 - X_2 )}_1 \leq \gamma^m \norm{X_1 - X_2}_1 \,.
\end{align*}
\end{lemma}

\begin{proof}
Let $\{v_1, ..., v_{n}\}$ be the normalized left eigenvectors of $H$ corresponding to ordered eigenvalues $\{\lambda_1, ..., \lambda_n\}$. Then $v_1 = \nu$, $\lambda_1 = 1$, and for all $i \geq 2$, we have that $\lambda_i < 1$ (since the chain is ergodic) and $v_i^\top {\bf 1} = 0$. Write $X_1$ in terms of the eigenvector basis as:
\begin{align*}
    X_1 = \alpha_1 \nu + \sum_{i=2}^n \alpha_i v_i \quad \text{and} \quad
    X_2 = \beta_1 \nu + \sum_{i=2}^n \beta_i v_i \,.
\end{align*}
Since $X_1^\top {\bf 1} = 1$ and $X_2^\top {\bf 1} = 1$, it is easy to see that $\alpha_1 = \beta_1 = 1$. Thus we have 
\begin{align*}
    \norm{H^\top (X_1 - X_2)}_1
    = \norm{ H^\top \sum_{i=2}^n (\alpha_i - \beta_i) v_i }_1
    \leq \gamma \norm{ \sum_{i=2}^n (\alpha_i - \beta_i) v_i}_1 = \gamma \norm{X_1 - X_2}_1 
\end{align*}
where the inequality follows from the definition of the ergodicity coefficient and the fact that ${\bf 1}^\top v_i =0$ for all $i \geq 2$. Since 
\begin{align*}
    {\bf 1}^\top H^\top \sum_{i=2}^n (\alpha_i - \beta_i) v_i = {\bf 1}^\top \sum_{i=2}^n \lambda_i (\alpha_i - \beta_i) v_i = 0,
\end{align*}
the inequality also holds for powers of $H$.
\end{proof}

\begin{lemma}[Doob martingale]
\label{lem:doob}
Let Assumption~\ref{ass:mixing} hold, 
and
let $\{(x_t,a_t)\}_{t=1}^T$ 
be the state-action sequence obtained when following policies $\pi_1, ..., \pi_k$ for $\tau$ steps each from an initial distribution $\nu_0$. For $t\in [T]$,
let $X_t$ be a binary indicator vector with a non-zero element at the linear index of the state-action pair $(x_t, a_t)$.
Define for $i \in [T]$,
\begin{align*}
B_i & = \E \left[ \sum_{t=1}^T X_t | X_1, ..., X_i \right], \quad \text{ and }\,\,
B_0  = \E \left[ \sum_{t=1}^T X_t\right].
\end{align*}
Then, $\{B_i\}_{i=0}^T$ is a vector-valued martingale: $\E[B_i-B_{i-1}|B_0,\dots,B_{i-1}]=0$ for $i=1,\dots,T$,  
and $\norm{B_i - B_{i-1}}_1 \le 2 (1 - \gamma)^{-1}$ holds for $i\in [T]$. 
\end{lemma}
The constructed martingale is known as the Doob martingale underlying the sum $\sum_{t=1}^T X_t$.

\begin{proof}
That $\{B_i\}_{i=0}^T$ is a martingale follows from the definition. We now bound its difference sequence. Let $H_t$ be the state-action transition matrix at time $t$, and let $H_{i:t} =\prod_{j=i}^{t-1} H_j$, and define $H_{i:i} = I$.
Then, for $t=0, \dots, T-1$, $\E[X_{t+1} | X_t] = H_{t}^\top X_t$ and by the Markov property, for any $i\in [T]$,
\begin{align*}
B_i & 
= \sum_{t=1}^i X_t + \sum_{t=i+1}^{T} \E[X_t | X_i] 
= \sum_{t=1}^i X_t + \sum_{t=i+1}^{T} H_{i:t}^{\top} X_i,  \quad \text{and} \quad
B_0  = \sum_{t=1}^T  H_{0:t}^\top X_0.
\end{align*}
For any $i\in[T]$,
\begin{align}
    B_i - B_{i-1} &= \sum_{t=1}^i X_t - \sum_{t=1}^{i-1} X_t + \sum_{t=i+1}^{T} H_{i:t}^{\top} X_i - \sum_{t=i}^{T} H_{i-1:t}^\top X_{i-1}  \notag \\
    & = \sum_{t=i}^{T} H_{i:t}^{\top} (X_i - H_{i-1}^\top X_{i-1}). \label{eq:diffseq}
\end{align}

Since $X_i$ and $H_{i-1}^\top X_{i-1}$ are distribution vectors, under Assumption \ref{ass:mixing} and using Lemma~\ref{lem:contraction},
\begin{align*}
   \norm{B_i - B_{i-1}}_1 
   \leq \sum_{t=i}^T \norm{ H_{i:t}^{\top} (X_i - H_{i-1}^\top X_{i-1})}_1 
   \leq 2 \sum_{j=0}^{T-i} \gamma^j \leq 2 (1 - \gamma)^{-1} \,.
\end{align*}
\end{proof}

Let $(\cF_k)_k$ be a filtration and define $\E_k[\cdot] := \E[\cdot|\cF_k]$. We will make use of the following concentration results for the sum of random matrices and vectors. 
\begin{theorem}[Matrix Azuma, \citet{tropp2012user} Thm 7.1]
\label{thm:azuma}
Consider a finite $(\cF)_k$-adapted sequence $\{X_k\}$ of Hermitian matrices of dimension $m$, and a fixed sequence $\{A_k\}$ of Hermitian matrices that satisfy $\E_{k-1} X_k = 0$ and $X_k^2 \preceq A_k^2$ almost surely.  Let $v = \norm{\sum_k A_k^2}$.
Then with probability at least $1-\delta$, $ \norm{ \sum_k X_k }_2 \leq 2\sqrt{2v \ln (m / \delta) }$.
 \end{theorem}
A version of Theorem~\ref{thm:azuma} for non-Hermitian matrices of dimension $m_1 \times m_2$ can be obtained by applying the theorem to a Hermitian dilation of $X$,
  $\cD(X) = \big[\begin{smallmatrix}0 & X \\ X^* & 0 \end{smallmatrix} \big]$, which satisfies $\lambda_{\max}(\cD(X)) = \norm{X}$ and $\cD(X)^2 =  \big[\begin{smallmatrix}XX^* & 0 \\ 0 & X^*X \end{smallmatrix} \big]$.  In this case, we have that $v = \max \left(\norm{\sum_k X_k X_k^*}, \norm{\sum_k X_k^* X_k} \right)$.
  
\begin{lemma} [Hoeffding-type inequality for norm-subGaussian random vectors, \citet{jin2019short}]
\label{lemma:nsg}
Consider random vectors $X_1, \ldots, X_n \in \R^d$ and corresponding filtrations $\cF_i = \sigma(X_1, \ldots, X_i)$  $i \in [n]$, such that $X_i |\cF_{i-1}$ is zero-mean norm-subGaussian with $\sigma_i \in \cF_{i-1}$. That is:
\begin{align*} 
    \E[X_i |\cF_i] = 0, \quad P(\norm{X_i} \geq t | \cF_{i-1}) \leq 2 \exp(-t^2 / 2\sigma_i^2) \quad \forall t \in \R, \forall i \in [n].
\end{align*}
If the condition is satisfied for fixed $\{\sigma_i\}$, there exists a constant $c$ such that for any $\delta > 0$, with probability at least $1-\delta$,
\begin{align*}
    \norm{\sum_{i=1}^n X_i} \leq c \sqrt{\sum_{i=1}^n \sigma_i^2 \log(2d/\delta)} \,.
\end{align*}

\end{lemma}

\section{Bounding the Difference Between Empirical and Average Rewards}
\label{sec:vtbound}
In this section, we bound the second term in Equation~\ref{eqn:regret_dec}, corresponding to the difference between empirical and average rewards. 

\begin{lemma}
\label{lem:vtbound}
Let Assumption~\ref{ass:mixing} hold, and assume that $\tau \geq \frac{\log T}{2 \log (1/\gamma)}$ and that $r(x, a) \in [0, 1]$ for all $x, a$. Then, by choosing $\eta = \frac{\sqrt{8\log |\cA|}}{Q_{\max} \sqrt{K}}$, we have with probability at least $1-\delta$, 
\begin{align*}
    \sum_{k=1}^K \sum_{t=(k-1)\tau+1}^{k\tau} (r_t - J_{\pi_k}) \leq 2(1 - \gamma)^{-1} \sqrt{2T \log(2/\delta)}
+ 2\sqrt{T} + (1 - \gamma)^{-2} \sqrt{8K\log|\cA|} \,.
\end{align*}
\end{lemma}

\begin{proof}

Let $r$ denote the vector of rewards, and recall that $J_{\pi} = \nu_\pi^\top r$. Let $X_t$ be the indicator vector for the state-action pair at time $t$, as in Lemma~\ref{lem:doob}, and let $\nu_t = \E[X_t]$. We have the following:
\begin{align*}
    V_T &:= \sum_{k=1}^K \sum_{t=(k-1)\tau+1}^{k\tau} (r_t - J_{\pi_k}) 
    = \sum_{k=1}^K \sum_{t=(k-1)\tau+1}^{k\tau} r^\top (X_t - \nu_t + \nu_t - \nu_{\pi_k}) 
\end{align*}
We slightly abuse the notation above by letting $\nu_t$ denote the state-action distribution at time $t$, and $\nu_\pi$ the stationary distribution of policy $\pi$.
Let $\{B_i\}_{i=0}^T$ be the Doob martingale in Lemma~\ref{lem:doob}. Then $B_0 =\sum_{t=1}^T \nu_t$ and $B_T = \sum_{t=1}^T X_t$, and the first term can be expressed as 
\begin{align*}
V_{T1} := \sum_{t=1}^T r^\top (X_t - \nu_t) = r^\top (B_T - B_0).
\end{align*}
By Lemma~\ref{lem:doob}, 
   $ |\langle B_i - B_{i-1}, r\rangle| \leq \norm{B_i - B_{i-1}}_1\norm{r}_\infty \leq 2 (1 -\gamma)^{-1} $.
Hence by Azuma's inequality, with probability at least $1-\delta$, 
\begin{equation}\label{eq:vt1}
    V_{T1} \leq 2(1 - \gamma)^{-1} \sqrt{2T \log(2/{\delta})}.
\end{equation}
For the second term we have 
\begin{align*}
    V_{T2} & := \sum_{k=1}^K \sum_{t=(k-1)\tau+1}^{k\tau} r^\top (\nu_t - \nu_{\pi_k}) \\
    & = \sum_{k=1}^K r^\top \bigg(\sum_{i=1}^\tau (H_{\pi_k}^i)^\top \nu_{(k-1)\tau} - \nu_{\pi_k}\bigg) \\
     & \leq \sum_{k=1}^K \norm{r}_\infty  \sum_{i=1}^\tau \left\| (H_{\pi_k}^i)^\top (\nu_{(k-1)\tau} - \nu_{\pi_{k-1}} + \nu_{\pi_{k-1}}) - \nu_{\pi_k} \right\|_1 \\
    & \leq \sum_{k=1}^K \sum_{i=1}^\tau \norm{ \nu_{(k-1)\tau} - \nu_{\pi_{k-1}} }_1 + \norm{(H_{\pi_k}^i)^\top \nu_{\pi_{k-1}} - \nu_{\pi_k}}_1  \\
    &  \leq \sum_{k=1}^K \sum_{i=1}^\tau \norm{ (H_{\pi_{(k-1)}}^\tau)^\top \nu_{(k-2)\tau} - \nu_{\pi_{k-1}} }_1 + \gamma^{i}  \norm{\nu_{\pi_{k-1}} - \nu_{\pi_k}}_1  \\
    &  \leq 2T \gamma^{\tau} + \frac{1}{1-\gamma} \sum_{k=1}^K \norm{\nu_{\pi_k} -\nu_{\pi_{k-1}}}_1 \,.
\end{align*}
For $\tau \geq  \frac{\log T}{2 \log (1/\gamma)}$, the first term is upper-bounded by $2\sqrt{T}$. 

Using results on perturbations of Markov chains \citep{seneta1988perturbation,cho2001comparison}, we have that
\begin{align*}
    \norm{\nu_{\pi_k} - \nu_{\pi_{k-1}}}_1 
    \leq \frac{1}{1 - \gamma} \norm{H_{\pi_k} - H_{\pi_{k-1}}}_\infty \leq \frac{1}{1 - \gamma} \max_x \norm{\pi_{k}(\cdot|x) - \pi_{k-1}(\cdot|x)}_1.
\end{align*}
Note that the policies $\pi_k(\cdot|x)$ are generated by running mirror descent on reward functions $\widehat{Q}_{\pi_k}(x, \cdot)$. 
A well-known property of mirror descent updates with entropy regularization (or equivalently, the exponentially-weighted-average algorithm) is that the difference between consecutive policies is bounded as
\begin{align*}
    \norm{\pi_{k+1}(\cdot|x) - \pi_k(\cdot|x)}_1 \leq \eta \norm{\widehat{Q}_{\pi_k}(x,\cdot)}_\infty \,.
\end{align*}
See e.g. \citet{NeGySzA13} Section V.A for a proof, which involves applying Pinsker's inequality and Hoeffding's lemma (\citet{cesa2006prediction} 
Section A.2 and Lemma A.6). 
Since we assume that $\norm{\widehat{Q}_{\pi_k}}_\infty \leq Q_{\max}$, we can obtain
\[
V_{T2} \leq 2\sqrt{T} + (1 - \gamma)^{-2} K \eta Q_{\max}.
\]
By choosing $\eta = \frac{\sqrt{8\log |\cA|}}{Q_{\max} \sqrt{K}}$, we can bound the second term as
\begin{equation}\label{eq:vt2}
    V_{T2} \leq 2\sqrt{T} + (1 - \gamma)^{-2} \sqrt{8K\log|\cA| }. 
\end{equation}

Putting Eq.~\eqref{eq:vt1} and~\eqref{eq:vt2} together, we obtain that with probability at least $1-\delta$, \[
V_T \leq 2(1 - \gamma)^{-1} \sqrt{2T \log(2/\delta)}
+ 2\sqrt{T} + (1 - \gamma)^{-2} \sqrt{8K\log|\cA|} \,.
\]
\end{proof}

\section{Proof of Lemma~\ref{lem:linear_est}}
\label{app:linear}

\begin{proof}
Recall that we split each phase into $2m$ blocks of size $b$ and let $\cH_i$ and $\cT_i$ denote the starting indices of odd and even blocks, respectively. 
We let $R_t$ denote the empirical $b$-step returns from the state action pair $(x_t, a_t)$ in phase $i$:
\begin{align*}
    R_t = \sum_{i=t}^{t+b} (r_i - \widehat J_{\pi_i}), \quad \widehat{J}_{\pi_i} = \frac{1}{|\cT_i|} \sum_{t \in \cT_i} r_t.
\end{align*}
We start by bounding the error in $R_t$. Let $X$ be a binary indicator vector for a state-action pair $(x, a)$. Let $H_{\pi}$ be the state-action transition kernel for policy $\pi$, and let $\nu_{\pi}$ be the corresponding stationary state-action distribution. We can write the action-value function at $(x, a)$ as
\begin{align*}
     Q_{\pi}(x, a) & = r(x, a) - J_{\pi} + X^\top H_\pi Q_{\pi} \\
     &= (X - \nu_\pi)^\top r + X^\top H_{\pi}( r - J_\pi {\bf 1} + H_{\pi} Q_\pi) \\
     & = \sum_{i=0}^\infty (X - \nu_\pi)^\top H_{\pi}^i r \,.
\end{align*}
Let $Q_{\pi}^b(x, a) = \sum_{i=0}^b (X - \nu_\pi)^\top H_{\pi}^i r$ be a version of $Q_{\pi}$ truncated to $b$ steps. Under uniform mixing, the difference to the true $Q_\pi$ is bounded as
\begin{align}
     |Q_{\pi}(x, a) - Q_{\pi}^b(x, a)| \leq \sum_{i=1}^\infty \left|(X - \nu_\pi)^\top H_{\pi}^{i+b} r \right| \leq \frac{2 \gamma^{b+1}}{1-\gamma} \,.
\end{align}
Let $b_t = Q_{\pi_i}^b(x_t, a_t) - Q_{\pi_i}(x_t, a_t)$ denote the truncation bias at time $t$, and let $z_t = \sum_{i=t}^{t+b}r_i - X_t^\top H_{\pi_i}^{(i-t)}r$ denote the reward noise. We will write
 \begin{align*}
     R_t = Q_{\pi_i}(x_t, a_t) + b(J_{\pi_i} - \widehat J_{\pi_i}) + z_t + b_t.
 \end{align*}

Note that $m= |\cH_i|$ and let 
\[\widehat M_i = \frac{1}{m} \sum_{t \in \cH_i} \phi_t \phi_t^\top + \frac{\alpha}{m} I \,. \]

We estimate the value function of each policy $\pi_i$ using data from phase $i$ as
\begin{align*}
    \widehat w_{\pi_i} 
    &= \widehat M_i^{-1} m^{-1} \sum_{t\in \cH_i} \phi_t R_t \\
    &= \widehat M_i^{-1} m^{-1} \sum_{t\in \cH_i} \phi_t (\phi_t^\top w_{\pi_i} + b_t + z_t + b (J_{\pi_i} - \widehat J_{\pi_i})) + \widehat M_i^{-1} \frac{\alpha}{m} (w_{\pi_i} -  w_{\pi_i}) \\
    &= w_{\pi_i} + \widehat M_i^{-1} m^{-1} \sum_{t\in \cH_i} \phi_t (z_t + b_t + b (J_{\pi_i} - \widehat J_{\pi_i})) - \widehat M_i^{-1} m^{-1} \alpha w_{\pi_i}
\end{align*}

Our estimate $\widehat w_k$ of $w_k = \frac{1}{k} \sum_{i=1}^k w_{\pi_i}$ can thus be written as follows:
\begin{align*}
    \widehat w_k - w_k &= \frac{1}{km} \sum_{i=1}^k  \sum_{t \in \cH_i} \widehat M_i^{-1}\phi_t (z_t + b_t +  b (J_{\pi_i} - \widehat J_{\pi_i}))
    - \frac{\alpha}{km} \sum_{i=1}^k \widehat M_i^{-1} w_{\pi_i}.
\end{align*}
We proceed to upper-bound the norm of the RHS. 

Set $\alpha = \sqrt{m/k}$. Let $C_w$ be an upper-bound on the norm of the true value-function weights $\norm{w_{\pi_i}}_2$ for $i=1,..., K$. In Appendix~\ref{app:m_bound}, we show that with probability at least $1-\delta$, for $m \geq 72 C_{\Phi}^4 \sigma^{-2} (1 -\gamma)^{-2} \log (d/\delta)$, $\norm{\widehat{M}_i^{-1}}_2 \leq 2\sigma^{-2}$.  Thus with probability at least $1-\delta$, the last error term is upper-bounded as
\begin{align}
\label{eq:bias_alpha}
 \frac{\alpha}{km} \left \| \sum_{k=1}^k \widehat{M}_i^{-1} w_{\pi_i}\right\|_2 \leq 2 \sigma^{-2}C_w (km)^{-1/2}.
\end{align}
Similarly, for
\begin{equation}\label{eq:b_condi}
   b \geq \frac{\log ((1-\gamma)^{-1}\sqrt{km})}{ \log (1/\gamma) },
\end{equation}
the norm of the truncation bias term is upper-bounded as 
\begin{align}
\label{eq:bias_truncation}
 \frac{1}{km} \sum_{i=1}^k  \sum_{t \in \cH_i} \norm{\widehat M_i^{-1}\phi_t  b_t }_2
 &\leq 
\frac{2 \gamma^b}{km (1 - \gamma)} \sum_{i=1}^k  \sum_{t \in \cH_i} \norm{\widehat M_i^{-1}\phi_t}_2 \leq 2 \sigma^{-2} C_\Phi  (km)^{-1/2}.
\end{align}
To bound the error terms corresponding to reward noise $z_t$ and average-error noise $J_{\pi_i} - \widehat J_{\pi_i}$, we rely on the independent blocks techniques of \citet{Yu94}. We show in Sections \ref{app:zt} and \ref{app:jbias} that with probability $1-2\delta$, for constants $c_1$ and $c_2$, each of these terms can be bounded as:
\begin{align*}
     \frac{1}{km} \left \| \sum_{i=1}^k  \sum_{t \in \cH_i} \widehat M_i^{-1}\phi_t  z_t \right \|_2 & \leq 2 c_1  C_\Phi \sigma^{-2} \sqrt{\frac{b \log(2d/\delta)}{km}} \\
     \frac{b}{km} \left \| \sum_{i=1}^k   (J_{\pi_i} - \widehat J_{\pi_i})  \sum_{t \in \cH_i} \widehat{M}_i^{-1} \phi_t    \right \|_2 &\leq 2 c_2  C_\Phi \sigma^{-2} b \sqrt{\frac{ \log(2d/\delta)}{km}}.
\end{align*}
Thus, putting terms together, we have for an absolute constant $c$, with probability at least $1-\delta$,
\begin{align*}
    \norm{\widehat w_k - w_k}_{2} \leq c  \sigma^{-2}(C_w + C_{\Phi})b \sqrt{\frac{ \log(2d/\delta)}{km}}.
\end{align*}
Note that this result holds for every $k\in[K]$ and thus also holds for $k=K$.
\end{proof}

\subsection{Bounding $\|\sum_{i=1}^k \widehat M_i^{-1} \sum_{t\in\cH_i} \phi_t  z_t\|_2$}
\label{app:zt}

Let $\norm{\cdot}_{\tv}$ denote the total variation norm.
\begin{definition}[$\beta$-mixing]
\label{def:betamix}
Let $\{Z_t\}_{t=1,2,\ldots}$ be a stochastic process.
Denote by $Z_{1:t}$ the collection $(Z_1,\ldots,Z_t)$, where we allow $t=\infty$.
Let $\sigma(Z_{i:j})$ denote the sigma-algebra generated by $Z_{i:j}$ ($i\le j$).
The $k^{\rm th}$ $\beta$-mixing coefficient of $\{Z_t\}$, $\beta_k$, is defined by
\begin{align*}
 \beta_k
 & = \sup_{t\ge 1} \EE{ \sup_{B\in\sigma(Z_{t+k:\infty})} |P(B|Z_{1:t})-P(B)| } \\
 & = \sup_{t\ge 1} \EE{ \norm{P_{Z_{t+k:\infty}|Z_{1:t}}(\cdot|Z_{1:t})-P_{Z_{t+k:\infty}}(\cdot)}_{\tv} }\,.
\end{align*}
$\{Z_t\}$ is said to be $\beta$-mixing if $\beta_k \ra 0$ as $k\ra\infty$.
In particular, we say that
 a $\beta$-mixing process mixes at an {\em exponential} rate with parameters $\obeta, \alpha , \gamma>0$
 if $\beta_k \le \obeta \exp(-\alpha k^\gamma)$ holds for all $k\ge 0$.
\end{definition}
Let $X_{t}$ be the indicator vector for the state-action pair $(x_t, a_t)$ as in Lemma~\ref{lem:doob}. Note that the distribution of $(x_{t+1}, a_{t+1})$ given $(x_t, a_t)$ can be written as $\E[X_{t+1} | X_t]$. Let $H_t$ be the state-action transition matrix at time $t$, let $H_{i:t} =\prod_{j=i}^{t-1} H_j$, and define $H_{i:i} = I$.  Then we have that $\E[X_{t+k} | X_{1:t}] = H_{t:t+k}^\top X_t$ and $\E[X_{t+k}] = H_{1:t+k}^\top \nu_0$, where $\nu_0$ is the initial state distribution. Thus, under the uniform mixing Assumption ~\ref{ass:mixing}, the $k^{th}$ $\beta$-mixing coefficient is bounded as:
\begin{align*}
    \beta_k  & \le \sup_{t\ge 1} \E  \sum_{j=k}^\infty \norm{H_{t:t+j}^\top X_t - H_{1:t+j}^\top \nu_0}_1 
    \leq \sup_{t\ge 1} \E \sum_{j=k}^\infty \gamma^j \norm{X_t - H_{1:t}^\top \nu_0}_1 \leq \frac{ 2 \gamma^k}{1 - \gamma} \,.
\end{align*}

We bound the noise terms using the independent blocks technique of \citet{Yu94}. 
Recall that we partition each phase into $2m$ blocks of size $b$. Thus, after $k$ phases we have a total of $2km$ blocks. Let $\bbP$ denote the joint distribution of state-action pairs in \emph{odd} blocks. 
Let $\cI_i$ denote the set of indices in the $i^{th}$ block, and let $x_{\cI_i}, a_{\cI_i}$ denote the corresponding states and actions. We factorize the joint distribution according to blocks:
\begin{align*}
    \bbP(x_{\cI_1},a_{\cI_1}, x_{\cI_3},a_{\cI_3},\ldots, x_{\cI_{2km-1}}, a_{\cI_{2km-1}}) 
     =& \;\; \bbP_1(x_{\cI_1},a_{\cI_1}) \times \bbP_3(x_{\cI_3},a_{\cI_3} | x_{\cI_1},a_{\cI_1}) \times 
    \cdots \\
&    \times
    \bbP_{2km - 1} (x_{\cI_{2km-1}},a_{\cI_{2km-1}} | x_{\cI_{2km-3}},a_{\cI_{2km-3}}).
\end{align*}
Let $\tilde{\bbP}_i$ be the marginal distribution over the variables in block $i$, and let $\tilde \bbP$ be the product of marginals of odd blocks. 

Corollary~2.7 of \citet{Yu94} implies that for any Borel-measurable set $E$,
\begin{equation}
\label{eq:indep-blocks}
  |\bbP(E) - \wt\bbP(E)| \leq (km-1)  \beta_b,
\end{equation}
where $\beta_b$ is the $b^{th}$ $\beta$-mixing coefficient of the process. 
The result follows since the size of the ``gap'' between successive blocks is $b$; see Appendix~\ref{app:ib} for more details.  

Recall that our estimates $\widehat w_{\pi_i}$ are based only on data in odd blocks in each phase. Let $\widetilde \E$ denote the expectation w.r.t. the product-of-marginals distribution $\tilde \bbP$. Then $\widetilde\E[\widehat M_i^{-1} \sum_{t\in \cH_i} \phi_t  z_t] =0$ because for $t\in \cH_i$ and under $\widetilde \bbP$, $z_t$ is zero-mean given $\phi_t$ and is independent of other feature vectors outside of the block. Furthermore, by Hoeffding's inequality $\tilde{\bbP}(|z_t|/b \geq a) \leq 2 \exp(-2ba^2)$.  
Since $\norm{\phi_t}_2 \leq C_\Phi$ and $\norm{\widehat M_i^{-1}}_2 \leq 2 \sigma^{-2}$ for large enough $m$, we have that 
\[
\tilde{\bbP}(\norm{\widehat M_i^{-1}\phi_t z_t}_2 \geq 2b \sigma^{-2} C_{\Phi} a) \leq 2\exp(-2 b a^2).
\]
Since $\widehat{M}_i^{-1}\phi_t z_t$ are norm-subGaussian vectors, using Lemma~\ref{lemma:nsg}, there exists a constant $c_1$ such that for any $\delta \geq 0$
\begin{align*}
    \tilde{\bbP}\left( \left\| \sum_{i=1}^k \widehat M_i^{-1} \sum_{t\in \cH_i} \phi_t  z_t \right\|_2 \geq 2 c_1  C_\Phi \sigma^{-2} \sqrt{b km \log(2d/\delta)} \right)  \leq \delta \,.
\end{align*}

Thus, using \eqref{eq:indep-blocks},
\[
\bbP \left(\left\| \sum_{i=1}^k \widehat M_i^{-1} \sum_{t\in \cH_i} \phi_t  z_t \right\|_2 \ge 2 c_1  C_\Phi \sigma^{-2} \sqrt{b km \log(2d/\delta)} \right) \le \delta + (km-1) \beta_b \;.
\]
Under Assumption~\ref{ass:mixing}, we have that $\beta_b \leq 2 \gamma^b (1 - \gamma)^{-1}$. Setting $\delta =2 km \gamma^b(1 - \gamma)^{-1}$ and solving for $b$ we get
\begin{align}\label{eq:b_condi_2}
b = \frac{\log ( 2km \delta^{-1}(1-\gamma)^{-1})}{\log (1/\gamma)}.
\end{align}
Notice that when $b$ is chosen as in Eq.~\eqref{eq:b_condi_2}, the condition~\eqref{eq:b_condi} is also satisfied.
Plugging this into the previous display gives that with probability at least $1-2\delta$,
\[
\left\| \sum_{i=1}^k \widehat M_i^{-1} \sum_{t\in \cH_i} \phi_t  z_t \right\|_2 \le 2 c_1  C_\Phi \sigma^{-2} \sqrt{b km \log(2d/\delta)}.
\]

\subsection{Bounding $\| \sum_{i=1}^k  \widehat M_i^{-1} \sum_{t\in \cH_i} \phi_t ( J_{\pi_i} - \widehat J_{\pi_i}) \|_2$}
\label{app:jbias}

Recall that the average-reward estimates $\widehat J_{\pi_i}$ are computed using time indices corresponding to the starts of even blocks, $\cT_i$. Thus this error term is only a function of the indices corresponding to block starts. 
Now let $\bbP$ denote the distribution over state-action pairs $(x_t, a_t)$ for indices $t$ corresponding to block starts, i.e.  $t \in \{1, b+1, 2b+1, ..., (2km -1)b+1\}$. We again factorize the distribution over blocks as $\bbP = \bbP_1 \otimes \bbP_2 \otimes \cdots \otimes \bbP_{2km} $.
Let $\tilde{\bbP} = \tilde{\bbP}_1 \otimes \tilde{\bbP}_2 \otimes \cdots \otimes \tilde{\bbP}_{2km} $ be a product-of-marginals distribution defined as follows. For odd $j$, let $\tilde \bbP_j$ be the marginal of $\bbP$ over $(x_{jb+1}, a_{jb+1})$. For even $j$ in phase $i$, let $\tilde \bbP_j = \nu_{\pi_i}$ correspond to the stationary distribution of the corresponding policy $\pi_i$. Using arguments similar to independent blocks, we show in Appendix~\ref{app:ib} that  
\[\norm{\bbP - \tilde{\bbP}}_1 \leq 2(2km - 1)\gamma^{b-1}.
\]
Let $\widetilde \E$ denote expectation w.r.t. the product-of-marginals distribution $\tilde \bbP$. 
Then $\widetilde\E [\widehat M_i^{-1} \sum_{t\in \cH^i} \phi_t  (J_{\pi_i} - \widehat J_{\pi_i})] =0$, since under $\tilde \bbP$, $\widehat J_{\pi_i}$ is the sum of rewards for state-action pairs distributed according to $\nu_{\pi_i}$, and these state-action pairs are independent of other data. 
Using a similar argument as in the previous section, for $b=1+\frac{\log(4km/\delta)}{\log(1/\gamma)}$, there exists a constant $c_2$ such that with probability at least $1-2\delta$,
\[
\left\| \sum_{i=1}^k \widehat M_i^{-1} \sum_{t\in \cH_i} \phi_t  (J_{\pi_i} - \widehat J_{\pi_i}) \right\|_2 \le 2 c_2 C_\Phi \sigma^{-2} \sqrt{km \log(2d/\delta)} \;.
\]

\subsection{Bounding $\|\widehat M_i^{-1}\|_2$}
\label{app:m_bound}
In this subsection, we show that with probability at least $1-\delta$, for $m \geq 72 C_{\Phi}^4 \sigma^{-2} (1 -\gamma)^{-2} \log (d/\delta))$, $\norm{M_i^{-1}}_2 \leq 2\sigma^{-2}$. 

Let $\Phi$ be a $|\cX||\cA| \times d$ matrix of all features. Let $D_i = {\rm diag}(\nu_{\pi_i})$,  and let $\widehat D_i = {\rm diag}(\sum_{t \in \cH_i} X_t )$, where $X_t$ is a state-action indicator as in Lemma~\ref{lem:doob}. Let $M_i = \Phi^\top D_i \Phi + \alpha m^{-1}  I$.
We can write $\widehat M_i^{-1}$ as
\begin{align*}
    \widehat M_i^{-1} &= ( \Phi^\top \widehat D_i \Phi + {\alpha}{\tau}^{-1}I + \Phi^\top (D_i - D_i) \Phi)^{-1} \\
    &= (M_i + \Phi^\top (D_i - D_i) \Phi)^{-1} \\
    &= (I + M_i^{-1}\Phi^\top (D_i - D_i) \Phi)^{-1} M_i^{-1}
\end{align*}

By Assumption~\ref{ass:excite} and \ref{ass:linear},  $\norm{M_i^{-1}}_2 \leq \sigma^{-2}$. In Appendix~\ref{app:matrix_azuma}, we show that w.p. at least $1-\delta$,
\begin{align*}
    \norm{\Phi^\top (\widehat D_i - D_i) \Phi}_2  
 \leq 6 m^{-1/2} C_{\Phi}^2 (1 - \gamma)^{-1} \sqrt{2\log (d /\delta)}
\end{align*}

Thus 
\begin{align*}
    \norm{\widehat M_i^{-1}}_2 \leq 
    \sigma^{-2} (1 - \sigma^{-2} 6 m^{-1/2} C_{\Phi}^2 (1 - \gamma)^{-1} \sqrt{2\log (d /\delta)})^{-1}
\end{align*}

For $m \geq 72 C_{\Phi}^4 \sigma^{-2} (1 -\gamma)^{-2} \log (d/\delta))$, the above norm is upper-bounded by $\norm{\widehat M_i^{-1}}_2 \leq 2 \sigma^{-2}$.

\subsection{Bounding $\|\Phi^\top (\widehat D_i - D_i) \Phi^\top\|_2$}
\label{app:matrix_azuma}

For any matrix $A$, 
\begin{align}
\label{eq:bounded_matrix}
\norm{\Phi^\top A \Phi}_2 = \bigg\| \sum_{ij} A_{ij} \phi_i \phi_j^\top \bigg\|_2 \leq \sum_{i, j} |A_{ij}| \norm{\phi_i \phi_j^\top}_2   \leq C_\Phi^2  \sum_{i, j} |A_{ij}|  = C_\Phi^2 \norm{A}_{1, 1} \,.
\end{align}
where $\norm{A}_{1, 1}$ denotes the sum of absolute entries of $A$.  Using the same notation for $X_t$  as in Lemma~\ref{lem:doob},
\begin{align*}
    \norm{\Phi^\top (\widehat D_i - D_i) \Phi}_2 
    & = \frac{1}{m}  \sum_{t \in \cH_i} \Phi^\top {\rm diag}(X_t - \nu_t + \nu_t -\nu_{\pi_i}) \Phi \\
    &\leq \frac{1}{m} \bigg\|   \sum_{t\in \cH_i} \Phi^\top {\rm diag}(X_t - \nu_t) \Phi\bigg\|_2 + \frac{C_\Psi^2}{m}\sum_{t\in \cH_i} \norm{\nu_t - \nu_{\pi_i}}_1 \,.
\end{align*}
Under the fast-mixing assumption~\ref{ass:mixing}, the second term is bounded by $2 C_{\Psi}^2 m^{-1} (1 - \gamma)^{-1}$.

For the first term, we can define a martingale $(B_i)_{i=0}^{m}$ similar to the Doob martingale in Lemma~\ref{lem:doob}, but defined only on the $m$ indices $\cH_i$. 
Note that $ \sum_{t \in \cH_i} \Phi^\top{\rm diag}(X_t - \nu_t) \Phi = \Phi^\top {\rm diag} (B_{m} - B_{0})\Phi $. Thus we can use matrix-Azuma to bound the difference sequence. Given that 
\begin{align*}
    \norm{ (\Phi^\top (B_i - B_{i-1}) \Phi)^2}_2 \leq 4 C_{\Phi}^4 (1 - \gamma)^{-2},
\end{align*}
combining the two terms, we have that with probability at least $1 - \delta$,
\begin{align*}
    \norm{\Phi^\top (\widehat D_i - D_i) \Phi}_2  &\leq 4 m^{-1/2} C_{\Phi}^2 (1 - \gamma)^{-1} \sqrt{2\log (d /\delta)} + 2 m^{-1} C_{\Phi}^2  (1 - \gamma)^{-1}  \\
    & \leq 6 m^{-1/2} C_{\Phi}^2 (1 - \gamma)^{-1} \sqrt{2\log (d /\delta)} \,.
\end{align*}

\section{Bounding $\E_{x \sim \mu*}[\widehat V_K(x) -  V_K(x)]$}
\label{app:linear_vpi}

We write the value function error as follows:
\begin{align*}
\E_{x \sim \mu*}[\widehat V_K(x) -  V_K(x)] 
&= \sum_x \mu_*(x) \sum_a \phi(x, a)^\top \frac{1}{K} \sum_{i=1}^K \pi_i(a|x) (\widehat w_{\pi_i} - w_{\pi_i}) \\
&\leq  \frac{1}{K} \sum_x \mu_*(x) \sum_a \norm{ \phi(x, a)}_2  \left \| \sum_{i=1}^K \pi_i(a|x) (\widehat w_{\pi_i} - w_{\pi_i}) \right\|_2
\end{align*}
Note that for any set of scalars $\{p_i\}_{i=1}^{K}$ with $p_i \in [0, 1]$, the term $\left \| \sum_{i=1}^K p_i (\widehat w_{\pi_i} - w_{\pi_i}) \right\|_2$ has the same upper bound as $\norm{\sum_{i=1}^K (\widehat w_{\pi_i} - w_{\pi_i})}_2$.  The reason is as follows. One part of the error includes bias terms \eqref{eq:bias_alpha} and \eqref{eq:bias_truncation}, whose upper bounds are only smaller when reweighted by scalars in $[0, 1]$. Thus we can simply upper-bound the bias by setting all $\{p_i\}_{i=1}^{K}$ to 1.  Another part of the error, analyzed in Appendices \ref{app:zt} and \ref{app:jbias} involves sums of norm-subGaussian vectors. In this case, applying the weights only results in these vectors potentially having smaller norm bounds. We keep the same bounds for simplicity, again corresponding to all $\{p_i\}_{i=1}^{K}$ equal to 1. Thus, reusing the results of the previous section, we have
\begin{align*}
    \E_{x \sim \mu*}[\widehat V_K(x) -  V_K(x)] 
&\leq  C_{\Phi} |\cA|  c  \sigma^{-2}(C_w + C_{\Phi})b \sqrt{\frac{ \log(2d/\delta)}{Km}}.
\end{align*}

\section{Independent Blocks}
\label{app:ib}

{\bf Blocks.} Recall that we partition each phase into $2m$ blocks of size $b$. Thus, after $k$ phases we have a total of $2km$ blocks. Let $\bbP$ denote the joint distribution of state-action pairs in odd blocks. 
Let $\cI_i$ denote the set of indices in the $i^{th}$ block, and let $x_{\cI_i}, a_{\cI_i}$ denote the corresponding states and actions. We factorize the joint distribution according to blocks:
\begin{align*}
    \bbP(x_{\cI_1},a_{\cI_1}, x_{\cI_3},a_{\cI_3},\ldots, x_{\cI_{2km-1}}, a_{\cI_{2km-1}}) 
     =& 
     \;\; \bbP_1(x_{\cI_1},a_{\cI_1}) \times \bbP_3(x_{\cI_3},a_{\cI_3} | x_{\cI_1},a_{\cI_1}) \times 
    \cdots \\
&    \times
    \bbP_{2km - 1} (x_{\cI_{2km-1}},a_{\cI_{2km-1}} | x_{\cI_{2km-3}},a_{\cI_{2km-3}}).
\end{align*}
Let $\tilde{\bbP}_i$ be the marginal distribution over the variables in block $i$, and let $\tilde \bbP$ be the product of marginals. Then the difference between the distributions $\tilde{\bbP}$ and $\bbP$ can be written as
\begin{align*}
     \bbP - \tilde \bbP  =&
\;\;     {\bbP_1} \otimes \bbP_3 \otimes \cdots \otimes \bbP_{2km-1} - {\bbP_1} \otimes \tilde{\bbP}_3 \cdots \otimes \tilde{\bbP}_{2km-1} \\
     =&
   \;\;  {\bbP_1} \otimes (\bbP_3 - \tilde{\bbP}_3) \otimes \bbP_5 \otimes \cdots \otimes\bbP_{2km-1}\\
&    + \bbP_1 \otimes \tilde \bbP_3  \otimes(\bbP_5 - \tilde{\bbP}_5) \otimes\bbP_7 \otimes\ldots \otimes\bbP_{2km-1} \\
   &  + \cdots \\
   &  +{\bbP}_1 \otimes \tilde{\bbP}_3 \otimes \tilde{\bbP}_5 \otimes \cdots \otimes \tilde{\bbP}_{2km-3} \otimes (\bbP_{2km-1} - \tilde{\bbP}_{2km-1}).
\end{align*}
Under $\beta$-mixing, since the gap between the blocks is of size $b$,  we have that 
\begin{align*}
    \norm{\bbP_i(x_{\cI_i},a_{\cI_i} | x_{\cI_{i-2}},a_{\cI_{i-2}}) - \tilde{\bbP}_i(x_{\cI_i},a_{\cI_i})}_1 \leq \beta_b = \frac{2 \gamma^b}{1-\gamma}.
\end{align*}
Thus the difference between the joint distribution and the product of marginals is bounded as 
\begin{align*}
    \norm{\bbP -\tilde{\bbP}}_1 \leq (km - 1) \beta_b.
\end{align*}

{\bf Block starts.} Now let $\bbP$ denote the distribution over state-action pairs $(x_t, a_t)$ for indices $t$ corresponding to block starts, i.e.  $t \in \{1, b+1, 2b+1, ..., (2km -1)b+1\}$. We again factorize the distribution over blocks:
\begin{align*}
    \bbP(x_1,a_1, x_{b+1}, a_{b+1}, \ldots, x_{(2km-1)b+1}, a_{(2km-1)b+1}) = \bbP_1(x_1, a_1) \prod_{j=2}^{2km} \bbP_i(x_{jb + 1}, a_{jb+1} | x_{(j-1)b+1}, a_{(j-1)b+1}). 
\end{align*}

Define a product-of-marginals distribution $\tilde \bbP = \tilde{\bbP}_1\otimes \tilde{\bbP}_2 \otimes \cdots \otimes \tilde{\bbP}_{2km}$ over the block-start variables as follows. For odd $j$, let $\tilde \bbP_j$ 
be the marginal of $\bbP$ over $(x_{jb+1}, a_{jb+1})$. For even $j$ in phase $i$, let $\tilde \bbP_j = \nu_{\pi_i}$ correspond to the stationary distribution of the policy $\pi_i$. 
Using the same notation as in Appendix~\ref{app:prelim}, 
let $X_t$ be the indicator vector for $(x_t, a_t)$
and let $H_{i:j}$ be the product of state-action transition matrices at times $i+1,..., j$. For odd blocks $j$, we have
\begin{align*}
    \norm{ \bbP_j(\cdot| x_{(j-1)b+1}, a_{(j-1)b+1}) - \tilde \bbP_j(\cdot)}_1 = \norm{ H_{(j-1)b+1:jb}^\top (X_{(j-1)b+1} - \tilde{\bbP}_{j-1})}_1 \leq 2 \gamma^{b-1} \,.
\end{align*}
Slightly abusing notation, let $H_{\pi_i}$ be the state-action transition matrix under policy $\pi_i$. For even blocks $j$ in phase $i$, since they always follow an odd block in the same phase,
\begin{align*}
    \norm{ \bbP_j(\cdot| x_{(j-1)b+1}, a_{(j-1)b+1}) - \tilde \bbP_j(\cdot)}_1 = \norm{ (H_{\pi_i}^{b-1})^\top( X_{(j-b)+1} - \nu_{\pi_i})}_1 \leq 2 \gamma^{b-1} \,.
\end{align*}

Thus, using a similar distribution decomposition as before, we have that $\norm{\bbP - \tilde{\bbP}}_1 \leq 2(2km-1) \gamma^{b-1}$.

\end{document}